%% file: main.tex
\pdftrailerid{}
\documentclass[pdflatex,iicol,sn-mathphys-num]{sn-jnl}
\makeatletter\g@addto@macro\abstractfont{\unboldmath}\makeatother
\usepackage{amsmath,amssymb,amsthm}
\usepackage{graphicx}
\usepackage{booktabs}
\usepackage{algorithm}
\usepackage{algpseudocode}
\usepackage{multirow}
\usepackage{placeins}
\input{macros}
\input{supplement_refs}
\begin{document}
\input{frontmatter}
\maketitle
\input{body}

\backmatter
\input{appendix}
\bibliography{references-rendered}
\input{declarations}
\end{document}


\maketitle
\begin{center}
\SupplementAuthorDetails
\end{center}
\input{supplement_body}

%% file: macros.tex
\providecommand{\pcite}{\citep}
\providecommand{\tcite}{\citet}
\theoremstyle{thmstyleone}
\newtheorem{proposition}{Proposition}[section]
\newtheorem{corollary}[proposition]{Corollary}
\newtheorem{lemma}{Lemma}[section]

%% file: supplement_refs.tex
\newcommand{\OnlineResourceRef}[1]{Online Resource~#1}

%% file: frontmatter.tex
\title{Understanding Latent-Dimension Scaling in Dynamical-System Learning through Spectral Reliability}
\author*[1]{\fnm{Itsushi} \sur{Sakata}}\email{itsushi.sakata@riken.jp}
\author[2]{\fnm{Yuta} \sur{Miyauchi}}\email{y-miyauchi@ist.osaka-u.ac.jp}
\author[2,1]{\fnm{Yoshinobu} \sur{Kawahara}}\email{kawahara@ist.osaka-u.ac.jp}
\affil*[1]{\orgname{RIKEN Center for Advanced Intelligence Project}, \orgaddress{\city{Tokyo}, \country{Japan}}}
\affil[2]{\orgdiv{Graduate School of Information Science and Technology}, \orgname{The University of Osaka}, \orgaddress{\city{Suita}, \state{Osaka}, \country{Japan}}}
\abstract{In deep learning, approximation theory motivates increasing representation size.
We ask whether this benefit extends to dynamics learning through autoregressive prediction.
We analyze the learned time evolution through the eigenstructure of Koopman operators, using relative residuals to detect spurious eigenpairs arising even as one-step error falls.
For bounded Koopman operators, we show that minimal residuals over learned dictionary spaces converge pointwise to their full-space counterparts as these spaces approximate the observable space in \(L^2\).
Our hypothesis is that Koopman spectral reliability helps explain how consistently rollout error decreases with increasing dimension.
We compare two models of a shared Koopman autoencoder trained alternately for reconstruction and latent evolution, using latent-prediction loss (one-step prediction errors in latent coordinates) or spectral-residual loss (relative residuals of candidate eigenpairs).
Across six chaotic systems, both models reduced median windowed rollout error from smallest to largest dimension.
The spectral-residual model achieved lower medians than the latent-prediction model for all systems and dimensions, and its median fell by a larger factor in every system.
Its median decreased monotonically with dimension in four systems, against one for latent prediction.
Against four baseline families, its mean valid prediction times were nearly always longer.
At the largest dimension under two-stage training, we compared eigenvalue positions with each learned dictionary's residual contours.
Spectral-residual eigenvalues concentrated in low-residual regions, whereas latent-prediction eigenvalues also appeared in high-residual regions, consistent with the hypothesis.}
\keywords{Koopman operator, Koopman autoencoder, Latent-dimension scaling, Autoregressive rollout, Pseudospectra, Chaotic dynamical systems}

%% file: body.tex
\input{introduction}
\input{problem_formulation}
\input{learning_methods}
\input{scaling_results}
\input{spectral_results}
\FloatBarrier
\input{conclusion}

%% file: introduction.tex
\section{Introduction}
\label{sec:introduction}
In deep learning, studies have reported power-law decreases in loss with model size, data volume, and computation across multiple domains and for language models \pcite{hestness2017scaling,kaplan2020scaling}, and examined how a fixed compute budget should be allocated between model parameters and training data \pcite{hoffmann2022training}.
\tcite{zhai2022scaling} have reported improvements from scaling in visual recognition with Vision Transformers.
For surrogate models of partial differential equations, a tendency for prediction error to decrease with increasing parameter count was reported \pcite{hao2024dpot,mccabe2024multiple,herde2024poseidon}; in weather forecasting with Stormer, larger models had lower forecast errors, and differences between models were greater at longer prediction horizons \pcite{nguyen2024scaling}.
World models use their predictions of future states or observations for action planning and policy learning \pcite{hafner2019learning,micheli2023transformers}, and in DreamerV3 and TD-MPC2, increasing parameter count improved control performance on the evaluated tasks \pcite{hafner2025mastering,hansen2024tdmpc2}.
In contrast, for zero-shot prediction of chaotic systems, no statistically significant difference in median valid prediction time was observed among the three Chronos models with the highest parameter counts \pcite{zhang2025zeroshot}.
Approximation theory characterizes the expressivity of neural networks.
Universal approximation \pcite{hornik1989} and approximation rates for rectified linear unit (ReLU) networks \pcite{yarotsky2017error,yarotsky2019phase,nakada2019adaptive} characterize the relationship between representation size and achievable approximation error under regularity and architecture assumptions.
Increasing the latent dimension increases the number of observables encoding the state and expands the family of representable dynamics.

When one-step predictors trained on true inputs feed their predictions back as inputs during autoregressive rollout, the input distribution changes; scheduled sampling and training with self-generated inputs address this change \pcite{bengio2015scheduledsampling,vlachas2023learning}.
To limit error accumulation, Model-Based Policy Optimization used short model rollouts starting from real-environment data for policy learning \pcite{janner2019trust}.
In model-based reinforcement learning, \tcite{asadi2018lipschitz} bound multi-step Wasserstein error by the one-step error bound multiplied by a sum of powers of the smaller Lipschitz constant of the true and learned transitions.
In control, \tcite{lambert2020objective} also reported cases where cumulative reward decreases even as one-step prediction likelihood improves.
We ask how the magnitude and consistency of improvement in autoregressive predictions with increasing representation size depend on the loss used to learn latent time evolution.
In this paper, scaling means the dependence of prediction error on representation size \(N\), and for the two Koopman autoencoder models, \(N\) sets the latent dimension and the hidden-layer width.

Existing approaches intervene in the latent transition model, training inputs, training objective, and constraints on the learned operator, and this paper focuses on the training objective.
\tcite{sanchezgonzalez2020gns} perturbed training inputs using noise, and \tcite{brandstetter2022mppde} using self-predictions.
\tcite{schiff2024dyslim} used measure-matching regularization for the long-term distribution, and \tcite{mamakoukas2023stable} imposed stability constraints on the learned operator.
\tcite{list2024unrolling} compared prediction errors across multiple network sizes for one-step training and unrolled training, which unfolds predictions through time.
With increasing latent dimension, \tcite{gonzalez2018deep} reported decreasing field prediction error, while \tcite{wu2022lepde} showed nonmonotonic multi-step prediction error.
\tcite{alkin2024upt} evaluated flow-field rollouts at multiple model sizes but did not evaluate whether the dependence of prediction error on model size varies with the loss function for the dynamics.

The Koopman operator represents nonlinear state evolution through a linear operator on a space of observables, and its spectrum is used for dynamical systems analysis \pcite{koopman1931hamiltonian,koopman1932continuous,mezic2005spectral}.
We use a Koopman autoencoder \pcite{otto2019linearly} that learns an encoder outputting a dictionary of observables, evolves the latent coordinates linearly in time, and reconstructs the state with a nonlinear decoder.
Since repeated application of a learned matrix generates latent dynamics, we evaluate them using rollout error and spectral properties.
Increasing the latent dimension \(N\), the number of dictionary elements, increases the size of the finite-dimensional approximation of the operator.

Spectral pollution occurs when eigenvalues of a finite approximation accumulate outside the spectrum of the Koopman operator even as the approximation size increases \pcite{colbrook2021rigorous}.
Residual dynamic mode decomposition (ResDMD) checks candidate eigenpairs obtained from finite approximations by evaluating their relative residuals with respect to the original Koopman operator \pcite{colbrook2021rigorous,colbrook2023residual}.
The relative residual measures the discrepancy between an observable's evolution and its multiplication by a candidate eigenvalue, normalized by the observable's magnitude.
Evaluation of empirical residuals from sample pairs \pcite{colbrook2024another} has been studied, and ResDMD has also been applied to systems with continuous spectra \pcite{sakata2024pseudoeigenfunctions}.
Methods have been proposed for learning dictionaries using operator residuals \pcite{xu2025nnresdmd}, for penalizing residuals and condition numbers \pcite{coote2026residualguided}, and for using a first-order approximation to a multi-step error bound as the training objective \pcite{conradie2026trustworthy}.
\tcite{conradie2026trustworthy} compare multi-step forecasts from full and truncated Koopman matrices and one-step errors as a function of the number of retained modes.
We evaluate the prediction error at each latent dimension and compare the positions of the estimated eigenvalues with the minimal residual for the learned dictionary.

\tcite{abuduweili2026scaling} augmented the state into the embedding and recovered it via a fixed linear projection, compared prediction error across latent dimensions for autonomous polynomial and robotic systems, and compared closed-loop control across combinations of latent dimensions and auxiliary-loss settings.
Whereas that study analyzed convergence rates for sampling and projection errors under assumptions on spectral decay and subspace approximation, our convergence analysis concerns the minimal residual for bounded Koopman operators and sequences of learned spaces.
We train a shared Koopman autoencoder using nonlinear reconstruction \pcite{otto2019linearly} in two stages and compare the Stage 2 losses: latent-prediction loss, measuring one-step prediction error in latent coordinates, and spectral-residual loss, measuring the empirical relative residuals of candidate eigenpairs.
With the spectral-residual loss, the dictionary is learned using residuals, as in the methods of \tcite{xu2025nnresdmd} and \tcite{coote2026residualguided}.
We hypothesize that the placement of eigenvalues in regions where empirical residuals are small (spectral reliability \pcite{coote2026residualguided}) helps explain how consistently the two models translate added dimensions into lower rollout error, and we expect the spectral-residual model to translate additional dimensions into more consistent and larger rollout improvements.

We make three contributions.
(i) We show that, under two-stage training, the spectral-residual model reduces median windowed rollout error more consistently with increasing latent dimension and by a larger factor from the smallest to the largest dimension than the latent-prediction model.
(ii) For bounded operators and sequences of learned spaces that asymptotically approximate every observable in \(L^2\), we establish that the minimal residual converges pointwise to the infimum of the relative residual over the whole observable space.
(iii) Under two-stage training, at the largest dimension in each of the six chaotic systems, the estimated Koopman eigenvalues of the spectral-residual model concentrated in low empirical-residual regions for the corresponding learned dictionaries, whereas latent-prediction eigenvalues also appeared in high-residual regions.

Section~\ref{sec:problem-formulation} states the prediction problem and convergence results.
Section~\ref{sec:learning-methods} describes the methods and experiments.
Sections~\ref{sec:scaling-results} and~\ref{sec:spectral-results} present prediction errors and spectral properties, and Sect.~\ref{sec:conclusion} concludes the paper.

%% file: problem_formulation.tex
\section{Problem formulation}
\label{sec:problem-formulation}

\subsection{Koopman operator and observables}
\label{sec:states-observables-model-size}
We take the \(d_x\)-dimensional observation vector \(\mathbf{x}_t\) in the state space \(\Omega\subseteq\mathbb{R}^{d_x}\) as the state and consider systems whose state evolves in discrete time steps according to a map \(F:\Omega\to\Omega\), that is, \(\mathbf{x}_{t+1}=F(\mathbf{x}_t)\).
The observation vector is a delay-coordinate vector of the measured state for systems with delay-coordinate inputs, and the measured state itself for the remaining systems.
A function \(g:\Omega\to\mathbb{C}\) of the state is called an observable.

Let \(\mu\) be a probability measure on \(\Omega\), and let \(\mathcal{H}=L^2(\Omega,\mu;\mathbb{C})\) be the Hilbert space of complex-valued square-integrable observables.
We assume that composition \(g\mapsto g\circ F\) defines a bounded linear operator \(\mathcal{K}:\mathcal{H}\to \mathcal{H}\) on \(L^2\) equivalence classes and denote its operator norm by \(\|\mathcal{K}\|\).
Even when \(F\) is nonlinear, \(\mathcal{K}\) is linear in the observable, and we analyze system behavior through the operator's eigenstructure and spectrum \pcite{koopman1932continuous,mezic2005spectral}.
\(\mathcal{K}\) is generally infinite-dimensional; it need not have a finite-dimensional invariant subspace beyond the one spanned by the constant observables and can have a continuous spectrum \pcite{mezic2005spectral,colbrook2021rigorous}.

For latent dimension \(N\in\mathbb{N}\), the encoder \(\phi_{\theta_{\mathrm{e}}}:\mathbb{R}^{d_x}\to\mathbb{R}^{N}\) and decoder \(\psi_{\theta_{\mathrm{d}}}:\mathbb{R}^{N}\to\mathbb{R}^{d_x}\) are maps parameterized by \(\theta_{\mathrm{e}}\) and \(\theta_{\mathrm{d}}\), respectively.
For fixed parameters and an estimated latent matrix \(\widehat{\mathbf{A}}\in\mathbb{R}^{N\times N}\), prediction is given by \(\widehat{\mathbf{z}}_0=\phi_{\theta_{\mathrm{e}}}(\mathbf{x}_0)\), \(\widehat{\mathbf{z}}_{t+1}=\widehat{\mathbf{A}}\widehat{\mathbf{z}}_t\), and \(\widehat{\mathbf{x}}_t=\psi_{\theta_{\mathrm{d}}}(\widehat{\mathbf{z}}_t)\).

On the state space \(\Omega\), we denote the encoder \(\phi_{\theta_{\mathrm{e}}}\) learned at each latent dimension \(N\) by \(\phi_N\), the learned dictionary.
We assume that the \(N\) coordinate observables of \(\phi_N\) belong to \(\mathcal{H}\) and let \(V_N\subset\mathcal{H}\) be the finite-dimensional subspace they span.
Any element of \(V_N\) can be written as \(\phi_N^\top \mathbf{a}\) for some \(\mathbf{a}\in\mathbb{C}^N\).
We assume that \(V_N\neq\{0\}\).

\subsection{Relative residuals and pseudospectra}
\label{sec:residuals-spectral-diagnostics}

For \(g\in \mathcal{H}\setminus\{0\}\) and \(\lambda\in\mathbb{C}\), following \tcite{colbrook2021rigorous}, we measure the accuracy of a candidate eigenpair \((\lambda,g)\) of the Koopman operator by the following relative residual:
\begin{equation}
    \label{eq:residual}
    \operatorname{res}(\lambda,g):=\frac{\|\mathcal{K}g-\lambda g\|_{\mathcal{H}}}{\|g\|_{\mathcal{H}}}.
\end{equation}
Equation~\eqref{eq:residual} tests a specified observable; minimizing over \(V_N\) tests whether this space contains a nonzero observable with a small residual at \(\lambda\).
For each \(\lambda\in\mathbb{C}\), define \pcite{colbrook2021rigorous,colbrook2023residual}
\begin{equation}
    \label{eq:restricted-residual-infima}
    \begin{aligned}
    \tau_N(\lambda)&:=\inf_{g\in V_N\setminus\{0\}}\operatorname{res}(\lambda,g),\\
    \tau(\lambda)&:=\inf_{g\in \mathcal{H}\setminus\{0\}}\operatorname{res}(\lambda,g).
    \end{aligned}
\end{equation}
In particular, \(\tau_N(\lambda)\leq\operatorname{res}(\lambda,g)\) for every \(g\in V_N\setminus\{0\}\).
Thus, a large \(\tau_N(\lambda)\) implies large residuals for all candidates in \(V_N\), whereas a small \(\tau_N(\lambda)\) implies that some candidate has a small residual but does not guarantee the accuracy of a particular computed eigenpair.
For \(\epsilon>0\), we define the restricted and full-space strict sublevel sets.
\begin{equation}
    \begin{aligned}
    \sigma_{N,\epsilon,\mathrm{ap}}(\mathcal{K})&:=\{\lambda\in\mathbb{C}:\tau_N(\lambda)<\epsilon\},\\
    \sigma_{\epsilon,\mathrm{ap}}(\mathcal{K})&:=\{\lambda\in\mathbb{C}:\tau(\lambda)<\epsilon\}.
    \end{aligned}
    \label{eq:spectral-sets}
\end{equation}
The approximate point \(\epsilon\)-pseudospectrum of \tcite{colbrook2021rigorous} is the closure of the full-space strict sublevel set \(\sigma_{\epsilon,\mathrm{ap}}(\mathcal{K})\).
The approximate point spectrum is \(\sigma_{\mathrm{ap}}(\mathcal{K})=\{\lambda\in\mathbb{C}:\tau(\lambda)=0\}\).

We also define the minimal residual based on \(M\) observed snapshot pairs \((\mathbf{x}^{(i)},\mathbf{y}^{(i)})\) of states one time step apart, with \(\mathbf{x}^{(i)}\in\Omega\), \(\mathbf{y}^{(i)}=F(\mathbf{x}^{(i)})\), and \(i=1,\dots,M\).
We choose a basis of \(V_N\), fix a measurable representative of each basis function, and define point evaluations using linear combinations of these representatives.
For given \(M,N\), assume that every nonzero \(g\in V_N\) satisfies \(\frac{1}{M}\sum_{i=1}^{M}|g(\mathbf{x}^{(i)})|^2>0\).
For \(\lambda\in\mathbb{C}\) and \(g\in V_N\setminus\{0\}\), we define the empirical relative residual \(\widehat{\operatorname{res}}_M(\lambda,g)\ge0\) computed from snapshot data as the nonnegative square root of the following expression:
\begin{equation}
    \label{eq:empirical-res}
    \begin{aligned}
    &\widehat{\operatorname{res}}_M(\lambda,g)^2\\
    &\quad:=
    \frac{\sum_{i=1}^{M}|g(\mathbf{y}^{(i)})-\lambda g(\mathbf{x}^{(i)})|^2}
         {\sum_{i=1}^{M}|g(\mathbf{x}^{(i)})|^2}.
    \end{aligned}
\end{equation}
We define the minimal residual over \(V_N\) by the following expression:
\begin{equation}
    \label{eq:empirical-tau}
    \widehat{\tau}_{M,N}(\lambda):=\min_{g\in V_N\setminus\{0\}}\widehat{\operatorname{res}}_M(\lambda,g).
\end{equation}

Let \(\widehat{\mathbf{A}}\) be a candidate finite-dimensional approximation of the Koopman operator on the learned dictionary, estimated using the procedure in Sect.~\ref{sec:model-learning}, and let \(\widehat{\tau}^{\mathrm{num}}_{M,N}\) be the minimal residual defined in Sect.~\ref{sec:metrics-analysis}.
Operationally, spectral reliability means that the eigenvalues of \(\widehat{\mathbf{A}}\) fall where \(\widehat{\tau}^{\mathrm{num}}_{M,N}\) is small, that is, where the retained numerical subspace of \(V_N\) contains a candidate whose empirical relative residual is small at that \(\lambda\).

\subsection{Convergence assumptions}
\label{sec:convergence-scope}
For each \(g\in \mathcal{H}\), define \(E_N(g)\) by
\begin{equation}
    E_N(g):=\inf_{v\in V_N}\|g-v\|_\mathcal{H}.
\end{equation}
\par (A1) Composition defines a bounded linear operator \(\mathcal{K}\) (if \(\mu\) is \(F\)-invariant, then \(\mathcal{K}\) is an isometry \pcite{colbrook2021rigorous}).
\par (A2) Each coordinate observable of \(\phi_N\) belongs to \(\mathcal{H}\), and \(V_N\neq\{0\}\).
The population Gram and cross-correlation matrices are defined by the \(\mathcal{H}\) inner products of the fixed basis and the action of \(\mathcal{K}\).
\par (A3) For the limit in latent dimension, we assume that \(E_N(g)\to0\) for every \(g\in \mathcal{H}\).
It allows non-nested sequences.
\par (A4) For the limit in sample size, we assume that, for each \(N\), the empirical Gram and cross-correlation matrices for a basis of \(V_N\) converge to the corresponding population matrices on a common set of full measure as \(M\to\infty\) (the sampling conditions and proof are given in Appendix~\ref{app:convergence}).

\begin{proposition}[Latent-dimension limit]
\label{prop:conv-latent}
Let \(V_N\subset\mathcal{H}\) be the complex linear span of the learned dictionary defined above, and assume that \(\mathcal{K}:\mathcal{H}\to \mathcal{H}\) is bounded.
Suppose that \(E_N(g)\to0\) as \(N\to\infty\) for every \(g\in \mathcal{H}\).
Then, for every \(g\in \mathcal{H}\) and \(\delta>0\), there exists \(N_\ast\in\mathbb{N}\) such that, for every \(N\ge N_\ast\), there exists \(v_N\in V_N\) satisfying
\begin{equation}
    \|g-v_N\|_\mathcal{H}+\|\mathcal{K}g-\mathcal{K}v_N\|_\mathcal{H}<\delta.
\end{equation}
For every \(\lambda\in\mathbb{C}\),
\begin{equation}
    \tau_N(\lambda)\to\tau(\lambda)
    \qquad (N\to\infty).
\end{equation}
Thus, for every \(\lambda\in\mathbb{C}\) and \(\epsilon>0\), if \(\tau(\lambda)<\epsilon\), then \(\lambda\in\sigma_{N,\epsilon,\mathrm{ap}}(\mathcal{K})\) for sufficiently large \(N\).
\end{proposition}

Building on \tcite{colbrook2021rigorous} and \tcite{colbrook2024limits}, this proposition shows convergence of the minimal residual for bounded Koopman operators under the assumption that the learned dictionary spaces, which need not be nested, approximate every observable (proof in Appendix~\ref{app:convergence}).
This convergence shows that, as the approximation error of each observable by the dictionary tends to zero, the minimal residual on the dictionary approaches the infimum of the relative residual over the whole observable space.

The inclusion \(V_N\subset \mathcal{H}\) gives \(\tau(\lambda)\le\tau_N(\lambda)\).
Boundedness implies that the \(L^2\) approximation in (A3) also holds in the graph norm \(\|g\|_\mathcal{H}+\|\mathcal{K}g\|_\mathcal{H}\), so a nonzero observable with residual arbitrarily close to \(\tau(\lambda)\) can be approximated by a sequence of elements \(v_N\in V_N\), one for each \(N\); the numerator and denominator of the residual then converge, with a positive denominator limit, giving \(\tau_N(\lambda)\to\tau(\lambda)\) (see Appendix~\ref{app:convergence} for the full proof).
This convergence is uniform on compact sets: for each \(N\) and any \(\lambda,\zeta\in\mathbb{C}\), \(|\tau_N(\lambda)-\tau_N(\zeta)|\le|\lambda-\zeta|\) holds (and the same holds for \(\tau\)), so the common Lipschitz bound strengthens pointwise convergence to uniform convergence on compact sets.

Under (A4), on a common set of full measure, \(\widehat{\tau}_{M,N}\) converges uniformly to \(\tau_N\) on every compact set \(\Lambda\subset\mathbb{C}\) (see the proof of Proposition~\ref{prop:conv-data} in Appendix~\ref{app:convergence}).
\begin{corollary}[Ordered limits]
\label{cor:ordered-limits}
Under Assumptions (A1)--(A4), on a set of full measure common to all \(N\), successive limits \pcite{colbrook2024limits} taken in the order \(M\to\infty\), \(N\to\infty\), \(\epsilon\downarrow0\) characterize the approximate point spectrum:
\begin{equation}
    \begin{aligned}
    \sigma_{\mathrm{ap}}(\mathcal{K})
    &=\bigcap_{\epsilon>0}\sigma_{\epsilon,\mathrm{ap}}(\mathcal{K})\\
    &=\Bigl\{\lambda\in\mathbb{C}:
      \lim_{N\to\infty}\lim_{M\to\infty}
      \widehat{\tau}_{M,N}(\lambda)=0\Bigr\}.
    \end{aligned}
\end{equation}
\end{corollary}
The proof is given as Proposition~\ref{prop:conv-epsilon} in Appendix~\ref{app:convergence}.
The next section describes learning dictionaries and candidate eigenpairs and computing the plotted minimal residual \(\widehat{\tau}^{\mathrm{num}}_{M,N}\).

%% file: learning_methods.tex
\section{Methods}
\label{sec:learning-methods}
The main comparison evaluates the model trained with latent-prediction loss and the model trained with spectral-residual loss within a shared two-stage Koopman autoencoder family as representation size increases.
Both models also enter comparisons with and without the operator-norm penalty, under two-stage and joint training, and with four baseline families.

\subsection{Model and training losses}
\label{sec:model-learning}
Because latent transition losses or Koopman regression losses alone admit zero or constant observables as trivial solutions, we include a reconstruction loss to limit the mapping of distant observation vectors to the same latent coordinates (Lemma~\ref{lem:reconstruction-excludes-collapse}, Appendix~\ref{app:convergence}).
To avoid this degeneracy, \tcite{takeishi2017learning} and \tcite{otto2019linearly} used reconstruction losses, while \tcite{li2017extended} included the state coordinates in the dictionary.
The encoder \(\phi_{\theta_{\mathrm{e}}}\) and the decoder \(\psi_{\theta_{\mathrm{d}}}\) of both models are fully connected multilayer perceptrons (MLPs) with hidden layers of width \(N\) and affine output layers.

Stage 1 updates the encoder and decoder using the reconstruction loss, and Stage 2 then holds the decoder fixed and updates the encoder using either the latent-prediction loss or the spectral-residual loss.
The reconstruction loss for \(n_{\mathrm{rec}}\) observation vectors \(\mathbf{x}^{(i)}\) is
\begin{equation}
\mathcal{L}_{\mathrm{rec}}(\theta_{\mathrm{e}},\theta_{\mathrm{d}})
=
\frac{1}{n_{\mathrm{rec}}d_x}\sum_{i=1}^{n_{\mathrm{rec}}}\!
\bigl\|\mathbf{x}^{(i)}-\psi_{\theta_{\mathrm{d}}}(\phi_{\theta_{\mathrm{e}}}(\mathbf{x}^{(i)}))\bigr\|_2^2.
\label{eq:two_stage_reconstruction_loss}
\end{equation}

\subsubsection{Operator estimation}
At each Stage 2 update, we draw an estimation minibatch from one pool of training pairs and a held-out minibatch from a second, disjoint pool.
The held-out minibatch is not used to estimate the operator at that update; it is used to evaluate the loss function for time evolution that updates the encoder \pcite{bertinetto2019metalearning,lee2019metalearning} (Appendix~\ref{app:settings}).
The estimation minibatch consists of \(m\) snapshot pairs \((\mathbf{x}^{(i)},\mathbf{y}^{(i)})\), \(i=1,\dots,m\), with \(\mathbf{y}^{(i)}=F(\mathbf{x}^{(i)})\).
The superscripts \(\top\) and \(*\) denote the transpose and conjugate transpose, respectively; \(\mathbf{I}_n\) denotes the \(n\times n\) identity matrix, and \(\|\cdot\|_F\) denotes the Frobenius norm.
Let \(\mathbf{Z}_X\) and \(\mathbf{Z}_Y\) be the latent snapshot matrices whose \(i\)th columns are \(\phi_N(\mathbf{x}^{(i)})\) and \(\phi_N(\mathbf{y}^{(i)})\), respectively:
\begin{equation}
\begin{aligned}
\mathbf{Z}_X
&=
[\phi_N(\mathbf{x}^{(1)}),\dots,\phi_N(\mathbf{x}^{(m)})]
\in\mathbb{R}^{N\times m},
\\
\mathbf{Z}_Y
&=
[\phi_N(\mathbf{y}^{(1)}),\dots,\phi_N(\mathbf{y}^{(m)})]
\in\mathbb{R}^{N\times m}.
\end{aligned}
\end{equation}
The matrices \(\mathbf{Z}'_X,\mathbf{Z}'_Y\in\mathbb{R}^{N\times m}\) are constructed analogously from the held-out minibatch.
Define the Gram matrices \(\mathbf{G}_{XX},\mathbf{G}_{YY}\) and the cross-correlation matrix \(\mathbf{G}_{XY}\)
\begin{equation}
\begin{aligned}
\mathbf{G}_{XX}&=\frac{1}{m}\mathbf{Z}_X\mathbf{Z}_X^\top,\qquad \mathbf{G}_{XY}=\frac{1}{m}\mathbf{Z}_X\mathbf{Z}_Y^\top,\\
\mathbf{G}_{YY}&=\frac{1}{m}\mathbf{Z}_Y\mathbf{Z}_Y^\top.
\end{aligned}
\end{equation}
For \(\|\mathbf{Z}_X\|_F>0\), the estimated latent matrix is the regularized least-squares solution
\begin{equation}
\begin{aligned}
\widehat{\mathbf{A}}&=\mathbf{Z}_Y\mathbf{Z}_X^\top\left(\mathbf{Z}_X\mathbf{Z}_X^\top+\varepsilon \mathbf{I}_N\right)^{-1}\\
&=\mathbf{G}_{XY}^\top\left(\mathbf{G}_{XX}+\varepsilon \mathbf{I}_N/m\right)^{-1}.
\end{aligned}
\label{eq:latent_operator}
\end{equation}
Because the Gram matrix becomes nearly singular when the latent coordinates are nearly linearly dependent, we use regularization to stabilize the least-squares solution \pcite{colbrook2021rigorous}.
The regularization parameter \(\varepsilon>0\) is also used as the threshold for retaining eigenvalues of \(\mathbf{Z}_X\mathbf{Z}_X^\top\), and its value is given in Appendix~\ref{app:settings}.

\subsubsection{Loss functions}
\textbf{Latent-prediction loss.}
The latent-prediction loss is
\begin{equation}
\mathcal{L}_{\mathrm{pred}}(\theta_{\mathrm{e}})
=
\frac{1}{Nm}
\left\|\mathbf{Z}'_Y-\widehat{\mathbf{A}}\mathbf{Z}'_X\right\|_F^2.
\label{eq:stage2_latent_prediction_loss}
\end{equation}

\noindent\textbf{Spectral-residual loss.}
For \(\mathbf{a}\in\mathbb{C}^N\) with \(\mathbf{a}^*\mathbf{G}_{XX}\mathbf{a}>0\), the empirical squared residual of \(\phi_N^\top \mathbf{a}\), corresponding to Eq.~\eqref{eq:residual} \pcite{colbrook2021rigorous,colbrook2023residual}, is
\begin{equation}
\begin{aligned}
&\frac{\mathbf{a}^*\left[\mathbf{G}_{YY}-\lambda \mathbf{G}_{XY}^*-\overline{\lambda}\,\mathbf{G}_{XY}+|\lambda|^2\mathbf{G}_{XX}\right]\mathbf{a}}{\mathbf{a}^*\mathbf{G}_{XX}\mathbf{a}}\\
&\qquad=\frac{\left\|\mathbf{a}^\top(\mathbf{Z}_Y-\lambda \mathbf{Z}_X)\right\|_2^2}{\left\|\mathbf{a}^\top \mathbf{Z}_X\right\|_2^2}.
\end{aligned}
\label{eq:empirical-relative-residual}
\end{equation}
For orthonormal eigenvectors \(\mathbf{u}_j\) of the matrix \(\mathbf{Z}_X\mathbf{Z}_X^\top=m\mathbf{G}_{XX}\) with eigenvalues \(\kappa_j\), let \(r\ge1\) be the number of eigenvalues satisfying \(\kappa_j\ge\varepsilon\):
\begin{equation}
\begin{aligned}
\mathbf{U}&=[\mathbf{u}_j:\ \kappa_j\ge\varepsilon]\in\mathbb{R}^{N\times r},\\
\widetilde{\mathbf{A}}&=\mathbf{U}^\top\widehat{\mathbf{A}}\mathbf{U}\in\mathbb{R}^{r\times r},\qquad r\ge1.
\end{aligned}
\label{eq:retained_basis}
\end{equation}
The transpose \(\widetilde{\mathbf{A}}^\top\) acts on observable coefficients.
We assume that \(\widetilde{\mathbf{A}}^\top\) is diagonalizable \pcite{colbrook2021rigorous,colbrook2024another}.
For its right eigenpairs \(\widetilde{\mathbf{A}}^\top \mathbf{v}_k=\lambda_k \mathbf{v}_k\), with \(\lambda_k\in\mathbb{C}\) and \(\mathbf{v}_k\in\mathbb{C}^r\setminus\{0\}\), set \(\mathbf{a}_k=\mathbf{U}\mathbf{v}_k\), which gives the candidate eigenfunction \(\phi_N^\top \mathbf{a}_k\in V_N\).
Normalization by the candidate eigenfunction's sample norm makes the residual scale-independent \pcite{colbrook2024another}.
Assume \(\|\mathbf{a}_k^\top \mathbf{Z}'_X\|_2>0\) for every retained eigenpair \(k=1,\ldots,r\).
Let \(\rho_k\ge0\) be the square root of the right-hand side of Eq.~\eqref{eq:empirical-relative-residual}, evaluated at \((\lambda,\mathbf{a})=(\lambda_k,\mathbf{a}_k)\) with the matrices \(\mathbf{Z}'_X,\mathbf{Z}'_Y\) in place of \(\mathbf{Z}_X,\mathbf{Z}_Y\).
We define the spectral-residual loss by
\begin{equation}
\mathcal{L}_{\mathrm{res}}(\theta_{\mathrm{e}})
=\frac{1}{r}\sum_{k=1}^{r}\rho_k^2.
\label{eq:stage2_spectral_residual_loss}
\end{equation}
ResKoopNet \pcite{xu2025nnresdmd} also learns the dictionary by minimizing the sum of squared empirical residuals over all computed eigenpairs.
When the retained subspace is the whole latent space, the numerator of \(\rho_k^2\) is the squared norm of the one-step prediction error in latent coordinates along \(\mathbf{a}_k\); when it is smaller, the numerator also includes contributions from coordinates outside the retained subspace.

\subsubsection{Training}
The loss function for time evolution \(\mathcal{L}_{\mathrm{evo}}\) is set to \(\mathcal{L}_{\mathrm{pred}}\) or \(\mathcal{L}_{\mathrm{res}}\) depending on the model, with the operator-norm penalty \(w_{\mathrm{op}}\|\widehat{\mathbf{A}}\|_2\) added when the penalty is used in auxiliary comparisons.
Starting from the initial parameters obtained from reconstruction pretraining, we run the two-stage training in Algorithm~\ref{alg:two-stage-training} for 200 epochs, allowing at most \(J_1=J_2=250\) updates per stage in each epoch and using the plateau stopping rule in Appendix~\ref{app:settings}.

\begin{algorithm}[!t]
  \caption{One epoch of two-stage training for the latent-prediction and spectral-residual models}
  \label{alg:two-stage-training}
  \begin{algorithmic}[1]
    \renewcommand{\algorithmicrequire}{\textbf{Input:}}
    \renewcommand{\algorithmicensure}{\textbf{Output:}}
    \Require Training pairs, current parameters $(\theta_{\mathrm{e}},\theta_{\mathrm{d}})$, Stage 2 loss $\mathcal{L}_{\mathrm{evo}}$, per-epoch update limits $(J_1,J_2)$
    \Ensure Updated parameters $(\theta_{\mathrm{e}},\theta_{\mathrm{d}})$
    \For{up to $J_1$ Stage 1 updates}
      \State Update $(\theta_{\mathrm{e}},\theta_{\mathrm{d}})$ with the reconstruction loss in Eq.~\eqref{eq:two_stage_reconstruction_loss}
    \EndFor
    \State Hold $\theta_{\mathrm{d}}$ fixed during Stage 2
    \For{up to $J_2$ Stage 2 updates}
      \State Draw an estimation minibatch and a held-out minibatch from the two disjoint pools of training pairs
      \State Form $(\mathbf{Z}_X,\mathbf{Z}_Y)$ and $(\mathbf{Z}'_X,\mathbf{Z}'_Y)$ from the respective minibatches using the current encoder
      \State Estimate $\widehat{\mathbf{A}}$ by Eq.~\eqref{eq:latent_operator}
      \State Update $\theta_{\mathrm{e}}$ with $\mathcal{L}_{\mathrm{evo}}$
    \EndFor
    \State \Return $(\theta_{\mathrm{e}},\theta_{\mathrm{d}})$
  \end{algorithmic}
\end{algorithm}

\subsubsection{Auxiliary comparisons}
Following prior Koopman models that learn a dictionary and regularize the finite-dimensional evolution matrix with Tikhonov, \(L_2\), or spectral-norm terms \pcite{li2017extended,yeung2019learning,terao2021extended,su2022convex,xiao2024deep}, we compare both models with and without the penalty \(w_{\mathrm{op}}\,\|\widehat{\mathbf{A}}\|_2\) for each system.
In the figure labels, Prediction and Residual denote the latent-prediction and spectral-residual models.

\subsubsection{Joint training}
Joint training uses the same operator estimate and the same loss function for time evolution, simultaneously updating the encoder and decoder:
\begin{equation}
\mathcal{L}_{\mathrm{joint}}(\theta_{\mathrm{e}},\theta_{\mathrm{d}})=w_{\mathrm{evo}}\,\mathcal{L}_{\mathrm{evo}}(\theta_{\mathrm{e}})+w_{\mathrm{rec}}\,\mathcal{L}_{\mathrm{rec}}(\theta_{\mathrm{e}},\theta_{\mathrm{d}}).
\end{equation}
The weights are \((w_{\mathrm{evo}},w_{\mathrm{rec}})=(2,1)\).
The number of epochs, learning rate, initialization, and model selection match those of two-stage training, and joint training allows at most 250 updates per epoch.

\subsection{Benchmark systems}
\label{sec:systems-data}
We evaluate both models on the six chaotic dynamical systems summarized in Table~\ref{tab:system-settings}.
These include systems where continuous or mixed Koopman spectral components can prevent closure with finitely many eigenfunctions \pcite{mezic2005spectral}.
The mixing Lorenz attractor is one such example \pcite{arbabi2017ergodic}.

For each of the five seeds, both models use the same six preprocessed trajectories, while initial conditions vary across seeds.
Here \(t\) is continuous time; all state and parameter symbols below are local.
The R\"ossler equations are
\begin{equation}
\dot{x}=-y-z, \qquad
\dot{y}=x+ay, \qquad
\dot{z}=b+z(x-c),
\end{equation}
with $(a,b,c)=(0.2,0.2,5.7)$.
The Lorenz-63 equations \pcite{lorenz1963deterministic} are
\begin{equation}
\dot{x}=\sigma(y-x), \qquad
\dot{y}=x(\rho-z)-y, \qquad
\dot{z}=xy-\beta z.
\end{equation}
Here, $(\sigma,\rho,\beta)=(10,28,8/3)$.
The Duffing dynamics are
\begin{equation}
\ddot{x}+\delta\dot{x}+\alpha x+\beta x^3=\gamma\cos(\omega t),
\end{equation}
with $(\delta,\alpha,\beta, \omega,\gamma)=(0.3,-1,1,1.2,0.5)$.
For Duffing, the measured state is \([x,\dot{x}]^\top\in\mathbb{R}^2\), where \(x\) is the scalar coordinate and \(\dot{x}\) is its velocity.
Mackey--Glass follows the delay differential equation \pcite{mackey1977oscillation}:
\begin{equation}
\dot{x}(t)=\beta_{\mathrm{MG}} \frac{x(t-\tau)}
{1+x(t-\tau)^{n_{\mathrm{MG}}}}
-\gamma_{\mathrm{MG}} x(t).
\end{equation}
Here, $(\gamma_{\mathrm{MG}},\beta_{\mathrm{MG}},\tau, n_{\mathrm{MG}})=(0.1,0.2,17,10)$.
For Lorenz-96 \pcite{lorenz2006predictability},
\begin{equation}
\begin{aligned}
\dot{x}_i&=(x_{i+1}-x_{i-2})x_{i-1}-x_i+F_{\mathrm{L96}},\\
&\qquad i=1,\dots,40,
\end{aligned}
\end{equation}
with periodic boundary conditions and forcing \(F_{\mathrm{L96}}=8\).
For the Kuramoto--Sivashinsky system, the scalar field \(u(x,t)\) on the periodic domain \(x\in[0,L)\) satisfies the following equation:
\begin{equation}
u_t + u_{xx} + u_{xxxx} + u u_x = 0.
\end{equation}

R\"ossler, Lorenz-63, Duffing, and Lorenz-96 are integrated with the explicit RK45 Runge--Kutta 5(4) method \pcite{dormand1980family}, with absolute and relative tolerances of \(10^{-12}\).
We numerically integrate Mackey--Glass using JiTCDDE \pcite{ansmann2018jitcdde}.
For Kuramoto--Sivashinsky, with \(L=22\), we use the fourth-order exponential time-differencing Runge--Kutta method ETDRK4 \pcite{kassam2005fourth} in Fourier space with a time step equal to the sampling interval, and apply \(2/3\)-rule dealiasing.
The grid is subsampled from \(128\) to \(64\) points by retaining every other grid point when the solution is sampled.

Let \(d\) be the dimension of the measured state, and let \(x_{t,j}\), \(j=1,\ldots,d\), denote its coordinates at time step \(t\).
For R\"ossler, Lorenz-63, Duffing, and Mackey--Glass, the delay-coordinate vector \(\mathbf{x}_t\) stacks \(q\) consecutive samples of the measured state, one time step apart, so that past state information is available for one-step prediction \pcite{brunton2017havok,das2019delay}:
\begin{equation}
\begin{aligned}
\mathbf{x}_t&=\bigl[x_{t-q+1,1},\ldots,x_{t-q+1,d},\ldots,\\
&\qquad x_{t,1},\ldots,x_{t,d}\bigr]^\top\in\mathbb{R}^{d_x},\\
d_x&=qd=N_0.
\end{aligned}
\end{equation}
For Duffing, the history carries the phase of the external periodic forcing.
For Mackey--Glass, whose state description requires a history, \(30\) samples span \(29\) time units of history for the delay \(\tau=17\).
Both models receive the same delay-coordinate input \(\mathbf{x}_t\) at every latent dimension \(N\).
For the high-dimensional Lorenz-96 and Kuramoto--Sivashinsky systems, we set \(q=1\), so \(\mathbf{x}_t\) is the measured state itself.
Using a common chronological split, we divide each raw trajectory into equal-length training/validation and test intervals, and construct delay-coordinate vectors within each interval.

Next, we divide the training/validation delay-coordinate vectors chronologically into a training set and a validation set in a \(9:1\) ratio, so that the nominal training:validation:test ratio is \(45:5:50\) \pcite{vlachas2020backpropagation,azencot2020forecasting}.
We center the resulting trajectories using the componentwise means of the training set and preserve the original variance scale.
Initial-condition generation is described in Appendix~\ref{app:settings}.
The sampling interval \(\Delta t\) is expressed in the time units of each system's governing equations.
With one Lyapunov time (LT) defined as the reciprocal of each system's largest Lyapunov exponent, the number of steps per LT in Table~\ref{tab:system-settings} was determined from numerical estimates of that exponent.
The evaluation window, prediction horizons, and valid prediction time are expressed in LT.

\begin{table*}[t]
\centering
\caption{Measured state dimension \(d\), input dimension \(N_0\) after delay coordinates, and simulation settings for the six systems.}
\label{tab:system-settings}
\setlength{\tabcolsep}{5pt}
\footnotesize
\begin{tabular}{@{}lrrrrrr@{}}
\toprule
System & \(d\) & \(N_0\) & \(\Delta t\) & Steps & Burn-in & Steps per LT \\
\midrule
R\"ossler & \(3\) & \(30\) & \(0.01\) & \(442{,}000\) & \(2{,}000\) & \(1{,}365.75\) \\
Lorenz-63 & \(3\) & \(30\) & \(0.01\) & \(202{,}000\) & \(2{,}000\) & \(110.75\) \\
Duffing & \(2\) & \(30\) & \(0.075\) & \(180{,}000\) & \(14{,}000\) & \(133\) \\
Mackey--Glass & \(1\) & \(30\) & \(1.0\) & \(112{,}000\) & \(12{,}000\) & \(175.2\) \\
Lorenz-96 & \(40\) & \(40\) & \(0.01\) & \(230{,}500\) & \(30{,}000\) & \(60.375\) \\
Kuramoto--Sivashinsky & \(64\) & \(64\) & \(0.25\) & \(88{,}000\) & \(8{,}000\) & \(70.5\) \\
\bottomrule
\end{tabular}
\end{table*}
\subsection{Evaluation metrics}
\label{sec:metrics-analysis}
\subsubsection{Prediction metrics}
Let \(t\) denote the prediction step from a single test starting point, with predicted and reference observation vectors \(\widehat{\mathbf{x}}_t,\mathbf{x}_t\).
Write \(\langle\cdot\rangle\) for the mean over the \(d_x\) components of the observation vector, set \(\bar{x}_t:=\langle\mathbf{x}_t\rangle\), and let \(\varepsilon_{\mathrm{vrmse}}=10^{-7}\).
For the Variance-Scaled Root Mean Squared Error (VRMSE), we apply The Well's definition for spatial fields to the components of observation vectors \pcite{ohana2024well}:
\begin{equation}
\begin{aligned}
&\operatorname{VRMSE}(\widehat{\mathbf{x}}_t,\mathbf{x}_t)\\
&\quad=
\left(
\frac{\left\langle |\widehat{\mathbf{x}}_t-\mathbf{x}_t|^2\right\rangle}
     {\left\langle |\mathbf{x}_t-\bar{x}_t|^2\right\rangle
      +\varepsilon_{\mathrm{vrmse}}}
\right)^{1/2}.
\end{aligned}
\label{eq:vrmse}
\end{equation}

For valid prediction time (VPT), let \(x_{t,j}\) and \(\widehat{x}_{t,j}\) denote the \(j\)th of the final \(d\) components of the reference observation vector \(\mathbf{x}_t\) and the prediction \(\widehat{\mathbf{x}}_t\) at prediction step \(t\), respectively.
We compute the temporal standard deviation \(\operatorname{std}(x_{\cdot,j})\) once from the time series \(x_{\cdot,j}\) over the reference test observation vectors and use the same value for all starting points.
With \(\nu\) denoting the number of steps per LT, let the prediction horizon be \(K=\lceil5\nu\rceil\) steps.
VPT is the number of consecutive prediction steps from the first prediction step for which the normalized error \(e_t\) stays below the threshold \(\epsilon_{\mathrm{VPT}}=0.5\), counted within this horizon and expressed in LT \pcite{vlachas2020backpropagation}:
\begin{align}
\operatorname{VPT}_{\epsilon_{\mathrm{VPT}}}
&=\frac{1}{\nu}\max\bigl\{n\in\{0,\ldots,K\}:\notag\\
&\qquad e_t<\epsilon_{\mathrm{VPT}}\ \text{for }t=1,\ldots,n\bigr\},\notag\\
e_t
&=\left[\frac{1}{d}\sum_{j=1}^{d}
\left(\frac{\widehat{x}_{t,j}-x_{t,j}}
{\operatorname{std}(x_{\cdot,j})}\right)^2\right]^{1/2}.
\label{eq:vpt}
\end{align}

For each of five seeds, we average VRMSE over test starting points and steps within the evaluation window, and average VPT over starting points after computing it at each one.
Across seeds, we summarize VRMSE by the median and interquartile range, and VPT by the mean and sample standard deviation (Table~\ref{tab:vpt-ablation}).
Rollout error is summarized by VRMSE over \((0.5,0.7]\) LT, where the differences in prediction accuracy between the models are clear.
VPT measures how long the normalized prediction error stays below the threshold within the specified prediction horizon.

For each system and representation size, we randomly sample \(3{,}000\) test starting points without replacement using a fixed random seed from positions where the entire prediction horizon lies within the test trajectory, and use the same set across all eight model variants and the four baselines.

\subsubsection{Spectral plots}
For the spectral reliability plots (Sect.~\ref{sec:rossler-spectra}), we pool all \(M\) consecutive pairs within the training, validation, and test splits, and obtain the retained subspace of the learned dictionary from a truncated singular value decomposition (SVD) of their weighted latent snapshot matrix \(M^{-1/2}[\,\phi_N(\mathbf{x}^{(i)})^\top\,]_{i=1}^{M}\).
For each \(\lambda\) on a grid in the complex plane, we compute \(\widehat{\tau}^{\mathrm{num}}_{M,N}(\lambda)\) by minimizing the empirical ResDMD relative residual over nonzero observables in this subspace.
This restriction gives \(\widehat{\tau}^{\mathrm{num}}_{M,N}(\lambda)\ge\widehat{\tau}_{M,N}(\lambda)\) whenever the latter is defined in Eq.~\eqref{eq:empirical-tau}.
The numerical settings for the residual contours are given in Appendix~\ref{app:settings}.
The sublevel set \(\{\lambda:\widehat{\tau}_{M,N}(\lambda)<\epsilon\}\), defined using the full dictionary space, is the ResDMD estimate of the approximate point \(\epsilon\)-pseudospectrum \pcite{colbrook2021rigorous}.
\subsection{Experimental setup}
\label{sec:comparison-training}
Following prior work \pcite{abuduweili2026scaling,wu2022lepde}, we sweep \(N\in\{N_0,2N_0,4N_0,8N_0,16N_0\}\) relative to each system's input dimension \(N_0=d_x\).
The operators used for validation and for the spectral plots are specified in Appendix~\ref{app:settings}.

After each epoch, we evaluate the current model and select the model with the lowest one-step prediction VRMSE on the validation interval \pcite{lusch2018deep,vlachas2020backpropagation,ohana2024well}.
Final metrics are computed on the test interval using the selected parameters \((\widehat{\theta}_{\mathrm{e}},\widehat{\theta}_{\mathrm{d}})\).

We use a latent-state neural ordinary differential equation (Neural ODE) \pcite{chen2018neuralode}, a Consistent Koopman Autoencoder (Consistent KAE) \pcite{azencot2020forecasting}, kernel dynamic mode decomposition (kernel DMD) \pcite{williams2015data,degennaro2019scalable} with random Fourier features (RFF) \pcite{rahimi2007random}, and an Echo State Network (ESN) \pcite{pathak2018model} as baselines.
The four baselines use the same raw trajectories, data splits, centering by the training-set means, five random seeds, VRMSE definition, and Lyapunov-time window as the main experiment.
All four baselines are trained, selected, and rolled out using centered inputs divided by the training-set standard deviations, and their autoregressive predictions from observed test starting points are transformed back to the centered original scale for evaluation.
\(N\) denotes the latent dimension for the two Koopman autoencoder models, the latent state dimension for Neural ODE, the autoencoder latent dimension for Consistent KAE, the number of RFF observables for kernel DMD, and the number of reservoir units for ESN.
The setting rules shared across systems are summarized in Appendix~\ref{app:settings}.

%% file: scaling_results.tex
\section{Prediction error across representation sizes}
\label{sec:scaling-results}

\subsection{Median error and variability}
\label{sec:error-scaling}
In the main two-stage comparison in Fig.~\ref{fig:latent-scaling}, the median windowed VRMSE of the spectral-residual model was lower than that of the latent-prediction model at all five evaluated latent dimensions for each of the six systems.
This median was below the latent-prediction first quartile in 29 of the 30 cases, except for Lorenz-96 at \(N=2N_0\).
For the reported rollouts, the spectral-residual model uses \(\widehat{\mathbf{A}}_{\mathrm{tr}}\) estimated from all training pairs, whereas the latent-prediction model uses the rank-constrained rollout operator derived from this matrix using the method in Appendix~\ref{app:settings} because rollouts with the matrix diverged.

\begin{figure*}[t]
    \centering
    \includegraphics[scale=1]{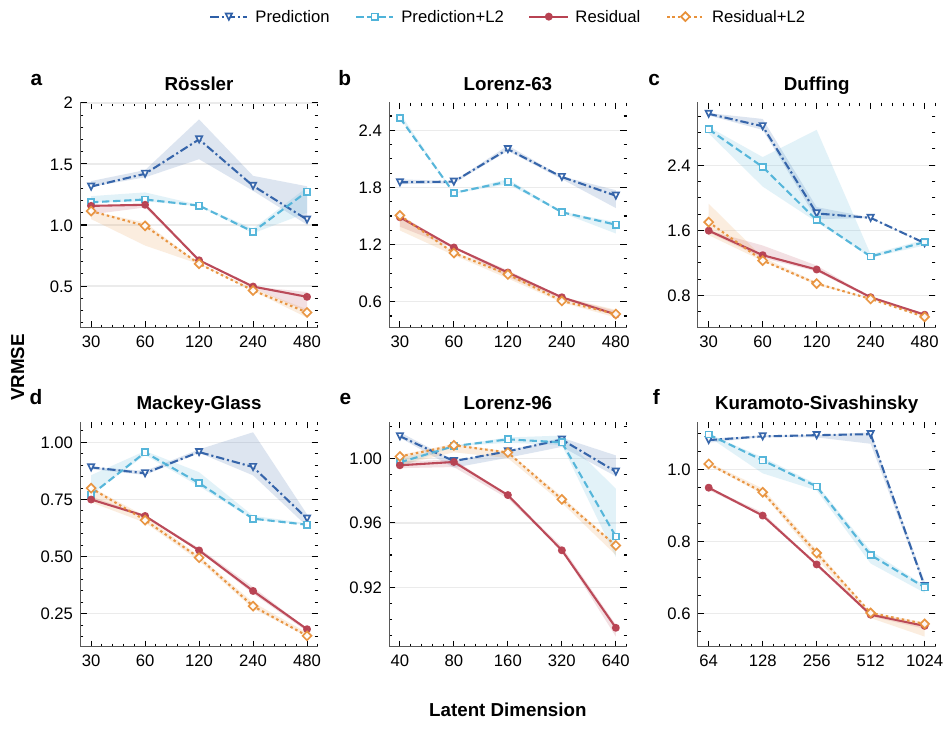}
\caption{Windowed autoregressive rollout VRMSE versus latent dimension for the two Stage 2 objectives in two-stage training (lower is better).
Prediction and Residual denote the latent-prediction and spectral-residual models, and +L2 adds the operator-norm penalty \(w_{\mathrm{op}}\|\widehat{\mathbf{A}}\|_2\).
Panels a--f show R\"ossler, Lorenz-63, Duffing, Mackey--Glass, Lorenz-96, and Kuramoto--Sivashinsky.
The evaluation window is \((0.5,0.7]\) LT, and the latent dimensions are 1, 2, 4, 8, and 16 times the input dimension \(N_0\) of each system.
For each seed, we average VRMSE over the test starting points and prediction steps within the evaluation window.
Lines show the median of these averages across five seeds, and bands show the interquartile range.
Each panel uses a linear vertical scale adjusted to its interquartile range}
    \label{fig:latent-scaling}
\end{figure*}

Under two-stage training, the spectral-residual reduction in median VRMSE from \(N_0\) to \(16N_0\), relative to each model's \(N_0\) median, ranged from \(10\%\) for Lorenz-96 to \(76\%\) for Mackey--Glass, compared with \(2\%\)--\(52\%\) for latent prediction; the reduction was larger for the spectral-residual model in each system.

The spectral-residual median decreased between every consecutive pair of dimensions for Lorenz-63, Duffing, Mackey--Glass, and Kuramoto--Sivashinsky; the exceptions were R\"ossler and Lorenz-96, where the median increased by less than \(1\%\) over the first interval before decreasing.
For the latent-prediction model, Duffing was the only system with a decrease between every consecutive pair of dimensions; in the other five systems, the median increased over 10 of the 20 intervals, by less than \(1\%\) over 6 of them and by up to \(20\%\).

For Kuramoto--Sivashinsky, the reductions from \(N_0\) to \(16N_0\) were comparable (\(40\%\) and \(38\%\) for the spectral-residual and latent-prediction models), but the spectral-residual median decreased over every interval, whereas the latent-prediction median stayed between 1.08 and 1.10 up to \(8N_0\) and fell sharply to 0.68 at \(16N_0\), and the absolute gap between the two medians was smaller at \(16N_0\) than at \(N_0\).

Under two-stage training with the operator-norm penalty applied to both models (+L2 in Fig.~\ref{fig:latent-scaling}), the spectral-residual model had a lower median VRMSE than the latent-prediction model in 27 of the 30 cases, except for Mackey--Glass at \(N=N_0\), and Lorenz-96 at \(N=N_0\) and \(N=2N_0\).

\subsection{Training variants and valid prediction time}
\label{sec:training-prediction-time}
In the VPT evaluation (Eq.~\eqref{eq:vpt}) \pcite{boffetta2002predictability,vlachas2020backpropagation,pathak2018model}, the spectral-residual model in the main two-stage comparison also achieved a longer mean VPT than the latent-prediction model at all five evaluated latent dimensions for each of the six systems, by factors from 1.06 to 5.21 (Table~\ref{tab:vpt-ablation}).
The spectral-residual mean VPT increased in all six systems under two-stage training, by factors from 1.39 to 5.07 from \(N_0\) to \(16N_0\).
It is nondecreasing across the evaluated latent dimensions in all six systems under both two-stage and joint training, with or without the penalty.

Over the first dimension interval, the spectral-residual median VRMSE increased slightly for R\"ossler and Lorenz-96, whereas mean VPT increased.

Lorenz-96 and Kuramoto--Sivashinsky were the two systems with the smallest reductions in median windowed VRMSE from \(N_0\) to \(16N_0\) for the spectral-residual model, but mean VPT for Kuramoto--Sivashinsky increased from 0.27 LT at \(N_0\) to 0.61 LT at \(16N_0\).
Mean VPT for Lorenz-96 was approximately 0.20--0.27 LT throughout the dimension sweep, shorter than the time to the start of the evaluation window \((0.5,0.7]\).

\input{table_vpt}

For R\"ossler at \(N=N_0\), mean VPT was longer for the spectral-residual model under two-stage training and for the latent-prediction model under joint training (Table~\ref{tab:vpt-ablation}).
At \(16N_0\), the mean VPT of the spectral-residual model was longer under two-stage training than under joint training in all six systems, and with the operator-norm penalty, it was longer in five systems, with Lorenz-96 as the exception.
Without an added operator-norm penalty, the mean VPT of the spectral-residual model was longer under two-stage training than under joint training in 23 of the 30 cases.

At \(16N_0\) with two-stage training, the penalty increased the mean VPT of the spectral-residual model in five systems (Table~\ref{tab:vpt-ablation}); mean VPT decreased for Lorenz-96.
For the latent-prediction model, mean VPT increased for Lorenz-63, Mackey--Glass, and Lorenz-96, decreased for R\"ossler, and changed little for Duffing and Kuramoto--Sivashinsky.
Under two-stage training with the penalty, the spectral-residual model also had a longer mean VPT than latent prediction at all five evaluated latent dimensions for each of the six systems.

\subsection{Comparison with baselines}
\label{sec:external-scaling}

In Fig.~\ref{fig:external-scaling}, the median windowed VRMSE of the spectral-residual model under two-stage training was lower than that of the four baselines (Sect.~\ref{sec:comparison-training}) at every comparison point, except for kernel DMD on R\"ossler at \(N=16N_0\) and Consistent KAE on Lorenz-96 at \(N=N_0\), \(N=2N_0\), and \(N=4N_0\).
At \(16N_0\), this median is lower than the lowest median among the four baseline families by factors from 1.08 to 1.76 in five systems; for R\"ossler, the kernel DMD median is lower by a factor of 1.46.

In Table~\ref{tab:vpt-ablation}, the baselines had shorter mean VPTs than the two-stage spectral-residual model, with eight exceptions.
In five of the eight exceptions, the difference in mean VPT was smaller than the sample standard deviation of the baseline's seed-level means across seeds.
Seven of the eight exceptions occurred at small representation sizes \(N\le 4N_0\); the remaining exception involved kernel DMD on R\"ossler at \(16N_0\).
On Lorenz-96, the mean VPT of kernel DMD is 0 at all five representation sizes.
On Lorenz-96 at \(N=N_0\), \(N=2N_0\), and \(N=4N_0\), Consistent KAE has lower median windowed VRMSE than the two-stage spectral-residual model, whereas the spectral-residual model has longer mean VPT.

In Fig.~\ref{fig:external-scaling}, the median windowed VRMSE of the two-stage spectral-residual model decreased monotonically with dimension in four systems, whereas no baseline family did so in more than three systems.

The following factors quantify changes from the smallest to the largest representation size: the kernel DMD median \pcite{kawahara2016kernel} decreases for R\"ossler, Lorenz-63, Duffing, and Mackey--Glass, by factors from 2.08 to 6.26, and is approximately constant for Lorenz-96 and Kuramoto--Sivashinsky.
The ESN median decreases for Mackey--Glass, by a factor of 2.16, and is approximately constant or nonmonotonic for the other systems.

The Neural ODE medians for R\"ossler and Mackey--Glass increase over the first interval and then decrease; overall, the former decreases by a factor of 1.80 and the latter increases by a factor of 1.35.
The median increases by factors of 1.36 for Lorenz-63 and 1.14 for Kuramoto--Sivashinsky, and decreases by factors of 1.21 for Duffing and 1.24 for Lorenz-96.

The Consistent KAE median decreases for R\"ossler, Lorenz-63, Mackey--Glass, and Kuramoto--Sivashinsky, by factors from 1.29 to 1.84, and is approximately constant for Duffing.
For Lorenz-96, it is between 0.96 and 0.97 up to \(4N_0\) and is 2.44 at \(8N_0\) and 1.60 at \(16N_0\).

The Gaussian radial basis function (RBF) kernel of kernel DMD is universal on compact input domains \pcite{steinwart2001influence,sriperumbudur2011universality}, and convergence of extended dynamic mode decomposition (EDMD) for a dense sequence of observables \pcite{korda2018convergence} is a population-level result.
This comparison concerns estimates from finite data at a fixed bandwidth (Appendix~\ref{app:settings}).

\begin{figure*}[tbp]
    \centering
    \includegraphics[scale=1]{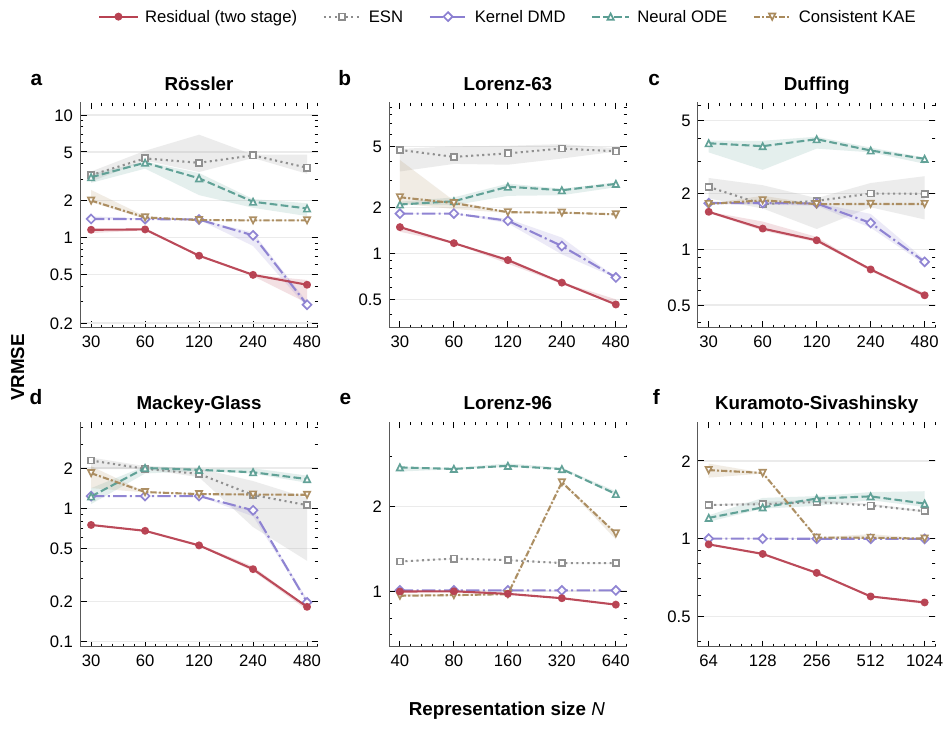}
\caption{Autoregressive rollout VRMSE versus representation size for the baseline families and the spectral-residual model (lower is better).
``Residual (two stage)'' denotes the spectral-residual model under two-stage training.
Panels a--f show R\"ossler, Lorenz-63, Duffing, Mackey--Glass, Lorenz-96, and Kuramoto--Sivashinsky.
The evaluation window is the same \((0.5,0.7]\) LT window as in Fig.~\ref{fig:latent-scaling}.
Markers show medians over five random seeds, and shaded bands show interquartile ranges.
The vertical axis is logarithmic}
    \label{fig:external-scaling}
\end{figure*}

Compared with the latent-prediction model in two-stage training, the spectral-residual model combined lower median windowed VRMSE with larger reductions from \(N_0\) to \(16N_0\) and more consistent improvement with increasing dimension.

%% file: table_vpt.tex
\begin{table*}[t]
\centering
\caption{VPT for 12 conditions in LT (longer is better).
Rows correspond to systems and representation sizes \(N\); columns correspond to the four two-stage and four joint model variants and the four baselines.
Prediction and Residual denote the latent-prediction and spectral-residual models, and +L2 adds the operator-norm penalty (Fig.~\ref{fig:latent-scaling}).
Section~\ref{sec:comparison-training} defines what \(N\) counts for each model family.
For each seed, VPT at a threshold of 0.5 is averaged over 3,000 test starting points; the table reports the mean \(\pm\) sample standard deviation of these seed-level means across five seeds.
Bold indicates the cell with the largest unrounded mean in each row, and a displayed value of 0.00 indicates an unrounded mean below 0.005 LT.}
\label{tab:vpt-ablation}
\setlength{\tabcolsep}{2pt}
\renewcommand{\arraystretch}{1.10}
\resizebox{\textwidth}{!}{%
\begin{tabular}{cr*{12}{c}}
\toprule
System & $N$ & \multicolumn{4}{c}{Two-stage} & \multicolumn{4}{c}{Joint} & \multicolumn{4}{c}{Baselines} \\
\cmidrule(lr){3-6}\cmidrule(lr){7-10}\cmidrule(lr){11-14}
 & & \shortstack{Prediction} & \shortstack{Prediction\\+L2} & \shortstack{Residual} & \shortstack{Residual\\+L2} & \shortstack{Prediction} & \shortstack{Prediction\\+L2} & \shortstack{Residual} & \shortstack{Residual\\+L2} & \shortstack{ESN} & \shortstack{Kernel\\DMD} & \shortstack{Neural\\ODE} & \shortstack{Consistent\\KAE} \\
\midrule
\multirow{5}{*}{R\"ossler} & 30 & 0.16$\pm$0.01 & 0.19$\pm$0.01 & 0.26$\pm$0.01 & 0.24$\pm$0.02 & \textbf{0.28$\pm$0.01} & 0.26$\pm$0.01 & 0.24$\pm$0.00 & 0.24$\pm$0.01 & 0.01$\pm$0.00 & 0.05$\pm$0.02 & 0.12$\pm$0.02 & 0.01$\pm$0.01 \\
 & 60 & 0.16$\pm$0.00 & 0.25$\pm$0.05 & 0.27$\pm$0.00 & \textbf{0.33$\pm$0.03} & 0.16$\pm$0.00 & 0.16$\pm$0.00 & 0.28$\pm$0.01 & 0.30$\pm$0.02 & 0.01$\pm$0.01 & 0.04$\pm$0.01 & 0.09$\pm$0.01 & 0.05$\pm$0.03 \\
 & 120 & 0.16$\pm$0.01 & 0.24$\pm$0.03 & 0.33$\pm$0.01 & \textbf{0.37$\pm$0.01} & 0.19$\pm$0.02 & 0.20$\pm$0.01 & 0.31$\pm$0.01 & 0.31$\pm$0.01 & 0.01$\pm$0.01 & 0.05$\pm$0.01 & 0.09$\pm$0.01 & 0.03$\pm$0.00 \\
 & 240 & 0.27$\pm$0.04 & 0.34$\pm$0.01 & 0.50$\pm$0.02 & \textbf{0.53$\pm$0.02} & 0.28$\pm$0.01 & 0.25$\pm$0.01 & 0.39$\pm$0.01 & 0.39$\pm$0.01 & 0.01$\pm$0.00 & 0.27$\pm$0.09 & 0.09$\pm$0.01 & 0.03$\pm$0.00 \\
 & 480 & 0.42$\pm$0.05 & 0.33$\pm$0.04 & 0.68$\pm$0.03 & \textbf{0.78$\pm$0.06} & 0.27$\pm$0.05 & 0.24$\pm$0.01 & 0.42$\pm$0.01 & 0.42$\pm$0.01 & 0.01$\pm$0.00 & 0.72$\pm$0.11 & 0.08$\pm$0.02 & 0.03$\pm$0.00 \\
\midrule
\multirow{5}{*}{Lorenz-63} & 30 & 0.19$\pm$0.01 & 0.21$\pm$0.01 & 0.34$\pm$0.01 & 0.38$\pm$0.02 & 0.20$\pm$0.01 & 0.22$\pm$0.00 & 0.34$\pm$0.00 & 0.35$\pm$0.00 & 0.12$\pm$0.03 & 0.04$\pm$0.02 & \textbf{0.50$\pm$0.04} & 0.04$\pm$0.02 \\
 & 60 & 0.20$\pm$0.01 & 0.39$\pm$0.06 & 0.44$\pm$0.01 & \textbf{0.51$\pm$0.01} & 0.20$\pm$0.01 & 0.22$\pm$0.00 & 0.45$\pm$0.01 & 0.45$\pm$0.00 & 0.15$\pm$0.04 & 0.06$\pm$0.03 & 0.48$\pm$0.06 & 0.12$\pm$0.03 \\
 & 120 & 0.13$\pm$0.00 & 0.37$\pm$0.02 & 0.58$\pm$0.01 & \textbf{0.61$\pm$0.01} & 0.21$\pm$0.01 & 0.33$\pm$0.06 & 0.54$\pm$0.01 & 0.54$\pm$0.00 & 0.14$\pm$0.05 & 0.13$\pm$0.05 & 0.42$\pm$0.03 & 0.08$\pm$0.01 \\
 & 240 & 0.25$\pm$0.02 & 0.39$\pm$0.01 & 0.78$\pm$0.03 & \textbf{0.82$\pm$0.04} & 0.20$\pm$0.00 & 0.45$\pm$0.02 & 0.67$\pm$0.01 & 0.70$\pm$0.07 & 0.12$\pm$0.02 & 0.52$\pm$0.04 & 0.44$\pm$0.04 & 0.05$\pm$0.00 \\
 & 480 & 0.37$\pm$0.01 & 0.43$\pm$0.06 & 1.01$\pm$0.01 & \textbf{1.06$\pm$0.03} & 0.35$\pm$0.03 & 0.41$\pm$0.01 & 0.81$\pm$0.01 & 0.88$\pm$0.14 & 0.14$\pm$0.02 & 0.84$\pm$0.02 & 0.34$\pm$0.02 & 0.05$\pm$0.00 \\
\midrule
\multirow{5}{*}{Duffing} & 30 & 0.11$\pm$0.00 & 0.18$\pm$0.02 & 0.29$\pm$0.02 & 0.29$\pm$0.05 & 0.13$\pm$0.01 & 0.15$\pm$0.00 & 0.25$\pm$0.00 & 0.25$\pm$0.00 & \textbf{0.40$\pm$0.15} & 0.03$\pm$0.01 & 0.33$\pm$0.03 & 0.07$\pm$0.04 \\
 & 60 & 0.12$\pm$0.03 & 0.21$\pm$0.00 & 0.34$\pm$0.02 & 0.38$\pm$0.02 & 0.11$\pm$0.00 & 0.20$\pm$0.00 & 0.30$\pm$0.00 & 0.29$\pm$0.00 & \textbf{0.38$\pm$0.11} & 0.05$\pm$0.03 & 0.33$\pm$0.02 & 0.13$\pm$0.04 \\
 & 120 & 0.19$\pm$0.02 & 0.20$\pm$0.01 & 0.45$\pm$0.04 & \textbf{0.65$\pm$0.06} & 0.19$\pm$0.00 & 0.20$\pm$0.00 & 0.39$\pm$0.01 & 0.40$\pm$0.01 & 0.50$\pm$0.14 & 0.09$\pm$0.04 & 0.32$\pm$0.03 & 0.08$\pm$0.02 \\
 & 240 & 0.19$\pm$0.01 & 0.40$\pm$0.02 & 1.01$\pm$0.08 & \textbf{1.13$\pm$0.11} & 0.21$\pm$0.00 & 0.18$\pm$0.00 & 0.62$\pm$0.01 & 0.61$\pm$0.01 & 0.47$\pm$0.17 & 0.38$\pm$0.19 & 0.33$\pm$0.04 & 0.04$\pm$0.00 \\
 & 480 & 0.30$\pm$0.02 & 0.30$\pm$0.02 & 1.46$\pm$0.08 & \textbf{1.55$\pm$0.09} & 0.33$\pm$0.00 & 0.33$\pm$0.00 & 1.04$\pm$0.04 & 1.05$\pm$0.04 & 0.42$\pm$0.17 & 0.97$\pm$0.08 & 0.34$\pm$0.03 & 0.03$\pm$0.00 \\
\midrule
\multirow{5}{*}{\shortstack[l]{Mackey--\\Glass}} & 30 & 0.07$\pm$0.01 & 0.20$\pm$0.00 & 0.20$\pm$0.03 & 0.22$\pm$0.01 & 0.06$\pm$0.01 & 0.08$\pm$0.00 & 0.22$\pm$0.00 & 0.22$\pm$0.00 & 0.11$\pm$0.02 & 0.01$\pm$0.00 & \textbf{0.43$\pm$0.07} & 0.04$\pm$0.01 \\
 & 60 & 0.11$\pm$0.00 & 0.30$\pm$0.00 & 0.38$\pm$0.03 & \textbf{0.42$\pm$0.03} & 0.11$\pm$0.00 & 0.18$\pm$0.00 & 0.36$\pm$0.00 & 0.35$\pm$0.00 & 0.18$\pm$0.02 & 0.02$\pm$0.00 & 0.32$\pm$0.02 & 0.06$\pm$0.01 \\
 & 120 & 0.16$\pm$0.04 & 0.25$\pm$0.02 & 0.55$\pm$0.02 & \textbf{0.59$\pm$0.02} & 0.20$\pm$0.01 & 0.24$\pm$0.03 & 0.46$\pm$0.00 & 0.46$\pm$0.00 & 0.27$\pm$0.06 & 0.03$\pm$0.02 & 0.31$\pm$0.01 & 0.05$\pm$0.01 \\
 & 240 & 0.24$\pm$0.02 & 0.40$\pm$0.01 & 0.64$\pm$0.03 & \textbf{0.74$\pm$0.02} & 0.22$\pm$0.01 & 0.33$\pm$0.07 & 0.60$\pm$0.00 & 0.60$\pm$0.00 & 0.52$\pm$0.13 & 0.21$\pm$0.05 & 0.30$\pm$0.02 & 0.03$\pm$0.00 \\
 & 480 & 0.40$\pm$0.03 & 0.42$\pm$0.02 & 0.92$\pm$0.05 & \textbf{0.96$\pm$0.03} & 0.37$\pm$0.02 & 0.37$\pm$0.02 & 0.78$\pm$0.01 & 0.78$\pm$0.01 & 0.66$\pm$0.33 & 0.89$\pm$0.02 & 0.26$\pm$0.01 & 0.02$\pm$0.00 \\
\midrule
\multirow{5}{*}{Lorenz-96} & 40 & 0.09$\pm$0.00 & 0.19$\pm$0.00 & \textbf{0.20$\pm$0.00} & 0.20$\pm$0.00 & 0.19$\pm$0.00 & 0.19$\pm$0.00 & 0.20$\pm$0.00 & 0.20$\pm$0.00 & 0.00$\pm$0.00 & 0.00$\pm$0.00 & 0.09$\pm$0.00 & 0.05$\pm$0.01 \\
 & 80 & 0.20$\pm$0.00 & 0.20$\pm$0.01 & 0.21$\pm$0.00 & 0.21$\pm$0.00 & 0.21$\pm$0.00 & 0.21$\pm$0.00 & \textbf{0.21$\pm$0.00} & 0.21$\pm$0.00 & 0.01$\pm$0.00 & 0.00$\pm$0.00 & 0.09$\pm$0.00 & 0.15$\pm$0.00 \\
 & 160 & 0.20$\pm$0.00 & 0.21$\pm$0.00 & \textbf{0.22$\pm$0.00} & 0.21$\pm$0.00 & 0.20$\pm$0.00 & 0.20$\pm$0.00 & 0.22$\pm$0.00 & 0.21$\pm$0.00 & 0.05$\pm$0.00 & 0.00$\pm$0.00 & 0.08$\pm$0.00 & 0.10$\pm$0.00 \\
 & 320 & 0.21$\pm$0.00 & 0.21$\pm$0.01 & 0.23$\pm$0.00 & 0.23$\pm$0.00 & 0.21$\pm$0.00 & 0.22$\pm$0.02 & \textbf{0.23$\pm$0.00} & 0.22$\pm$0.00 & 0.10$\pm$0.00 & 0.00$\pm$0.00 & 0.08$\pm$0.00 & 0.08$\pm$0.00 \\
 & 640 & 0.22$\pm$0.01 & 0.24$\pm$0.01 & 0.27$\pm$0.00 & 0.26$\pm$0.00 & 0.23$\pm$0.00 & 0.22$\pm$0.00 & 0.26$\pm$0.00 & \textbf{0.34$\pm$0.00} & 0.12$\pm$0.00 & 0.00$\pm$0.00 & 0.08$\pm$0.00 & 0.13$\pm$0.00 \\
\midrule
\multirow{5}{*}{\shortstack[l]{Kuramoto--\\Sivashinsky}} & 64 & 0.25$\pm$0.01 & 0.24$\pm$0.01 & 0.27$\pm$0.00 & 0.25$\pm$0.00 & 0.24$\pm$0.01 & 0.25$\pm$0.00 & \textbf{0.29$\pm$0.01} & 0.27$\pm$0.01 & 0.10$\pm$0.00 & 0.00$\pm$0.00 & 0.03$\pm$0.00 & 0.12$\pm$0.01 \\
 & 128 & 0.25$\pm$0.01 & 0.27$\pm$0.02 & \textbf{0.31$\pm$0.01} & 0.30$\pm$0.00 & 0.25$\pm$0.01 & 0.30$\pm$0.01 & 0.29$\pm$0.01 & 0.31$\pm$0.01 & 0.12$\pm$0.01 & 0.00$\pm$0.00 & 0.03$\pm$0.00 & 0.11$\pm$0.01 \\
 & 256 & 0.25$\pm$0.01 & 0.31$\pm$0.00 & \textbf{0.42$\pm$0.01} & 0.42$\pm$0.01 & 0.26$\pm$0.01 & 0.37$\pm$0.01 & 0.36$\pm$0.01 & 0.36$\pm$0.01 & 0.13$\pm$0.01 & 0.00$\pm$0.00 & 0.03$\pm$0.00 & 0.07$\pm$0.00 \\
 & 512 & 0.25$\pm$0.01 & 0.45$\pm$0.01 & 0.56$\pm$0.01 & \textbf{0.57$\pm$0.01} & 0.32$\pm$0.01 & 0.33$\pm$0.01 & 0.48$\pm$0.01 & 0.49$\pm$0.01 & 0.16$\pm$0.01 & 0.00$\pm$0.00 & 0.03$\pm$0.00 & 0.07$\pm$0.01 \\
 & 1024 & 0.50$\pm$0.01 & 0.50$\pm$0.01 & 0.61$\pm$0.01 & \textbf{0.62$\pm$0.02} & 0.51$\pm$0.01 & 0.46$\pm$0.01 & 0.57$\pm$0.01 & 0.57$\pm$0.01 & 0.18$\pm$0.01 & 0.00$\pm$0.00 & 0.06$\pm$0.00 & 0.03$\pm$0.00 \\
\bottomrule
\end{tabular}%
}
\end{table*}

%% file: spectral_results.tex
\section{Spectral properties and predictions}
\label{sec:spectral-results}

\subsection{Spectral analysis of the learned operators}
\label{sec:rossler-spectra}
Following ResDMD, we compare contours of the minimal residual \(\widehat{\tau}^{\mathrm{num}}_{M,N}\) over the retained subspace of each learned dictionary with the estimated Koopman eigenvalues for a posteriori verification \pcite{colbrook2021rigorous,xu2025nnresdmd}.
For the eight model variants at the largest evaluated latent dimension \(16N_0\), Fig.~\ref{fig:pseudospectra} shows the eigenvalues and contours of the minimal residual for R\"ossler, and the spectral plots and predictions for the other five systems are shown in \OnlineResourceRef{1}, Figs.~S1--S10.
The rollout matrix \(\widehat{\mathbf{A}}\) and \(\mathbf{z}_t=\phi_{\widehat{\theta}_{\mathrm{e}}}(\mathbf{x}_t)\) satisfy
\begin{equation*}
\mathbf{z}_n-\widehat{\mathbf{A}}^{n}\mathbf{z}_0
=\sum_{t=0}^{n-1}\widehat{\mathbf{A}}^{n-1-t}
\bigl(\mathbf{z}_{t+1}-\widehat{\mathbf{A}}\mathbf{z}_t\bigr).
\end{equation*}
Multi-step state prediction error depends on reconstruction error at the target state, one-step prediction errors in all latent directions, their propagation by powers of \(\widehat{\mathbf{A}}\), and the decoder.
The eigenvalues of \(\widehat{\mathbf{A}}_{\mathrm{tr}}\) are plotted for both models, whereas latent-prediction rollouts use the rank-constrained operator of Appendix~\ref{app:settings} to avoid divergence.

Under two-stage training, spectral-residual eigenvalues in all six systems were concentrated in regions where the plotted minimal residual was small.
By contrast, latent-prediction eigenvalues in both Prediction and Prediction+L2 also scattered into high-residual regions inside and outside the unit circle.
In the main two-stage comparison, the one-step prediction error of the latent-prediction model decreased from \(N_0\) to \(16N_0\) in all six systems.

Under joint training, spectral-residual eigenvalues extended farther toward the origin than those from two-stage training, and latent-prediction eigenvalues were also scattered across high-residual regions for their learned dictionary.
In the Prediction column for Lorenz-96, most eigenvalues clustered in low-residual regions near \(\lambda=1\) under both training methods (\OnlineResourceRef{1}, Fig.~S7).

\begin{figure*}[t]
    \centering
    \includegraphics[scale=1]{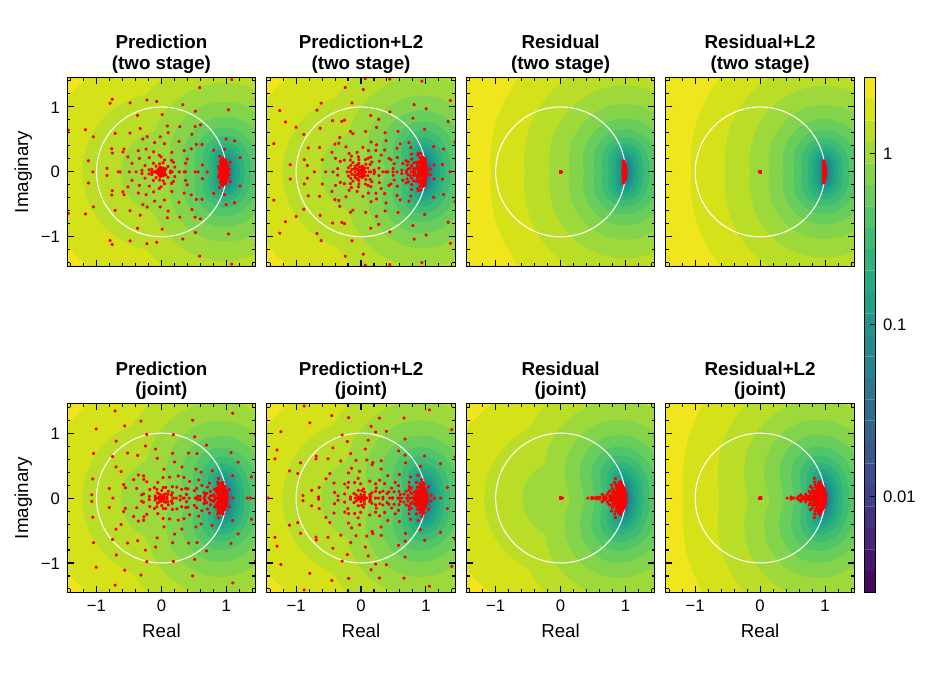}
\caption{Minimal residuals \(\widehat{\tau}^{\mathrm{num}}_{M,N}\) and eigenvalues for eight model variants in R\"ossler at latent dimension \(16N_0\).
Colors show, on a logarithmic scale, contours of the minimal residual computed over the retained numerical subspace, with the same color range across all panels.
Low-residual regions appear near \(\lambda=1\) in all eight panels.
White curves mark the unit circle, and red points mark the estimated Koopman eigenvalues within the plotted range.
The weighted latent snapshot matrix used for the contours is defined in Sect.~\ref{sec:metrics-analysis}, and the matrix whose eigenvalues are plotted is defined in Appendix~\ref{app:settings}.
The top row shows two-stage training, and the bottom row shows joint training.
From left to right, the columns show Prediction (the latent-prediction model), Prediction+L2, Residual (the spectral-residual model), and Residual+L2; +L2 adds the operator-norm penalty \(w_{\mathrm{op}}\|\widehat{\mathbf{A}}\|_2\)}
    \label{fig:pseudospectra}
\end{figure*}

\subsection{Qualitative comparison of predictions}
\label{sec:reconstructed-geometry}
We compare the four model variants separately for two-stage training and joint training.
Figure~\ref{fig:attractors} shows predictions made from different observation vectors at each prediction horizon.
For R\"ossler in Fig.~\ref{fig:attractors}, the spectral-residual model with two-stage training shows the reference fold at 0.6 LT, and at 1.0 LT it shows the main loop and a single excursion in the third coordinate that is lower than the highest excursions of the reference.
The overall shape is also captured when the operator-norm penalty is added.
For the spectral-residual model with joint training, the fold at 0.6 LT is lower than the highest excursions of the reference, and at 1.0 LT only the planar loop is visible.
For the latent-prediction model, the loop is distorted at 0.2 LT under both two-stage and joint training.

For Lorenz-63, Duffing, and Mackey--Glass, the spectral-residual model with two-stage training captures the overall shapes of the reference attractors at the displayed prediction horizons up to 1.0 LT.
For Lorenz-63, two lobes smaller than those of the reference attractor can be distinguished even at 1.0 LT, and small kinks appear in the curves.
For the spectral-residual model with joint training, the two lobes of Lorenz-63, the outer orbit and the loops on both sides in Duffing, and the folding structure of Mackey--Glass are visible at 1.0 LT, but all are smaller in extent than their reference counterparts and are distorted.
Compared with predictions from two-stage training, those from joint training have a noticeably smaller extent for Lorenz-63 and Duffing from 0.6 LT onward, while the difference in overall shape is small for Mackey--Glass.
For the latent-prediction model, distortions are present at 0.2 LT under both two-stage and joint training; for Mackey--Glass, pronounced loss of the reference folding structure occurs at 0.6 LT.
For Lorenz-96 and Kuramoto--Sivashinsky, the striped structure of the field is distorted for both models at 0.6 LT under both two-stage and joint training (\OnlineResourceRef{1}, Figs.~S8 and~S10).

\begin{figure*}[p]
    \centering
    \includegraphics[scale=1]{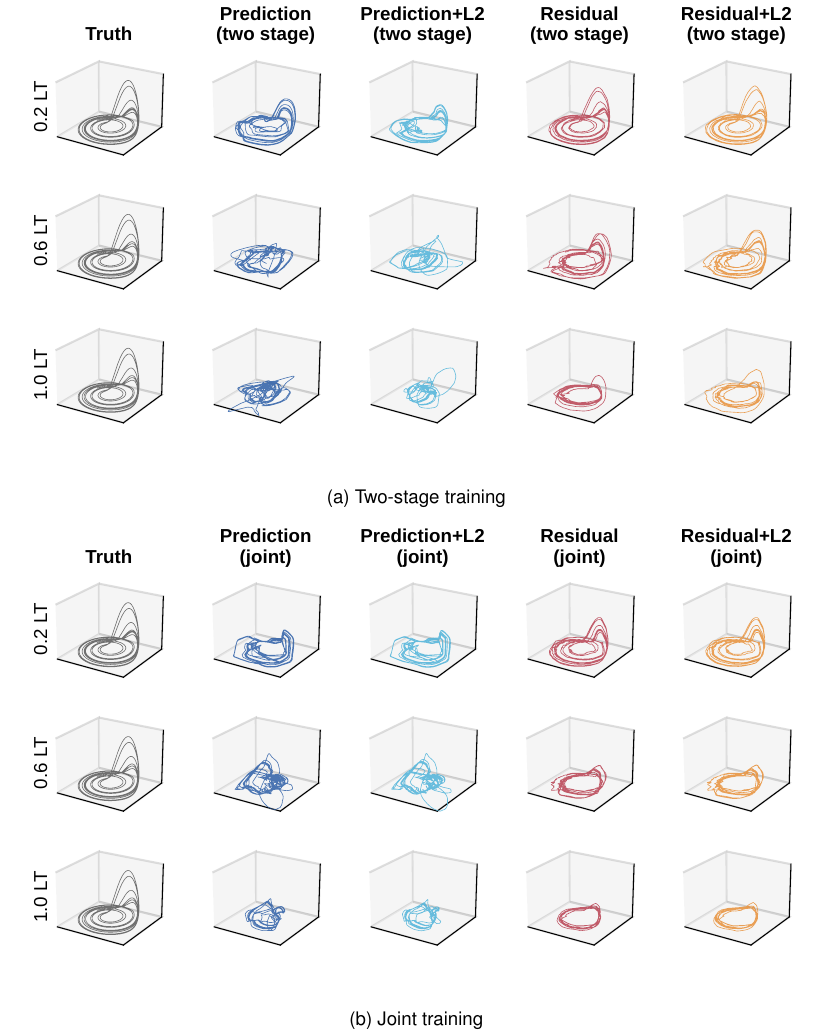}
\caption{Predictions for R\"ossler at latent dimension \(16N_0\) (top: two-stage training; bottom: joint training).
From left to right, the columns show the reference trajectory, Prediction (the latent-prediction model), Prediction+L2, Residual (the spectral-residual model), and Residual+L2; +L2 adds the operator-norm penalty \(w_{\mathrm{op}}\|\widehat{\mathbf{A}}\|_2\).
Each row corresponds to one prediction horizon (0.2, 0.6, or 1.0 LT), and predictions are made that far ahead from different observation vectors in the reference trajectory.
The plots show the three components of the measured state extracted from predictions whose target times fall within the first 5 LT of the test trajectory.
Axis limits follow the reference trajectory and are shared across all panels}
    \label{fig:attractors}
\end{figure*}

At \(16N_0\), the spectral-residual model under two-stage training had eigenvalues concentrated in low-residual regions for its learned dictionary, achieved a longer mean VPT than the latent-prediction model under two-stage training and the spectral-residual model under joint training, and achieved a lower median windowed VRMSE than the former.
This correspondence is consistent with an interpretation of this better prediction performance in terms of spectral reliability.

%% file: conclusion.tex
\section{Conclusions}
\label{sec:conclusion}
In the main two-stage comparison, the spectral-residual model reduced median windowed rollout error more consistently with increasing latent dimension than the latent-prediction model, and its error fell by a larger factor from the smallest to the largest dimension in every system.
Although both models reduced median error from the smallest to the largest dimension in all six chaotic systems, the spectral-residual model achieved lower median windowed errors and longer mean VPTs at all five evaluated latent dimensions in each system.
In the comparison with baseline families using family-specific training settings, the spectral-residual model also achieved favorable results on the same two metrics, with some exceptions.

In our theoretical analysis of bounded Koopman operators, we showed that as the space spanned by the learned dictionary approximates the observable space, the minimal residual over the dictionary space converges to the infimum of the relative residual over the whole observable space.
This result provides a mathematical basis for expecting residual-based spectral information to become more reliable as the latent dimension increases.

Under two-stage training at the largest evaluated dimension, the estimated Koopman eigenvalues of the spectral-residual model were concentrated in low-residual regions for its learned dictionary, whereas those of the latent-prediction model also appeared in high-residual regions.
The concentration of the eigenvalues of the spectral-residual model in low-residual regions for its learned dictionary, seen when they are overlaid on dictionary-specific residual contours, corresponds to longer mean VPT, and this correspondence is consistent with an interpretation of the prediction results in terms of spectral reliability.
Predictions from the spectral-residual model preserved the overall shapes of several reference attractors, whereas in field predictions for Lorenz-96 and Kuramoto--Sivashinsky, structures were distorted for both models despite the concentration of spectral-residual eigenvalues.

The next question is how spectral residuals, the propagation of one-step prediction errors in latent coordinates, and decoding jointly determine the dependence of rollout error on latent dimension.
For Koopman operators with continuous spectrum, a further question is under what conditions this dependence follows a power law, and what determines the exponent.

%% file: appendix.tex
\begin{appendices}
\renewcommand{\theHtable}{\Alph{section}\arabic{table}}
\input{content}
\end{appendices}

%% file: content.tex
\section{Spectral convergence and reconstruction}
\label{app:convergence}
\subsection*{Convergence of the minimal residual}
We take the limits in the order \(M\to\infty\), then \(N\to\infty\).
Joint or exchanged limits generally require additional assumptions relating \(M\) and \(N\) \pcite{colbrook2024limits}.

\subsubsection*{Setting}
We use \(\mathcal{H},\phi_N,V_N\), and \(\mathcal{K}\) as defined in Sect.~\ref{sec:states-observables-model-size}, and let \(I\) denote the identity operator on \(\mathcal{H}\).
For the data limit, we hold \(N\) and the space \(V_N\) fixed, set \(d_N=\dim V_N\le N\), and choose a basis \(b_1,\ldots,b_{d_N}\) independently of the evaluation samples.
Each \(g\in V_N\) can be written as \(g=\sum_{j=1}^{d_N}a_jb_j\) with a unique coefficient vector \(\mathbf{a}\in\mathbb{C}^{d_N}\), and we define point evaluations by linear combinations of fixed measurable representatives of the basis.
Define the evaluation matrices of the fixed basis, \(\boldsymbol{\Psi}_X,\boldsymbol{\Psi}_Y\in\mathbb{C}^{M\times d_N}\), and the Gram and cross-correlation matrices as follows:
\[
\begin{aligned}
(\boldsymbol{\Psi}_X)_{ij}&=b_j(\mathbf{x}^{(i)}),\\
(\boldsymbol{\Psi}_Y)_{ij}&=b_j(\mathbf{y}^{(i)}),
\end{aligned}
\qquad 1\le i\le M,\ 1\le j\le d_N,
\]
\[
\begin{aligned}
\mathbf{G}_{XX}^{(M,N)}&=\boldsymbol{\Psi}_X^\ast \mathbf{W}\boldsymbol{\Psi}_X,\\
\mathbf{G}_{XY}^{(M,N)}&=\boldsymbol{\Psi}_X^\ast \mathbf{W}\boldsymbol{\Psi}_Y,\\
\mathbf{G}_{YY}^{(M,N)}&=\boldsymbol{\Psi}_Y^\ast \mathbf{W}\boldsymbol{\Psi}_Y,\qquad \mathbf{W}=\mathbf{I}_M/M.
\end{aligned}
\]
When \(d_N=N\) and the coordinate observables of \(\phi_N\) are chosen in this order as the basis \(b_1,\ldots,b_N\), these matrices coincide with those in Sect.~\ref{sec:model-learning} if the same pairs are used with \(M=m\).
Positive definiteness of \(\mathbf{G}_{XX}^{(M,N)}\) is equivalent to the denominator positivity condition stated before Eq.~\eqref{eq:empirical-tau}, and under this condition the minimal residual has the following expression in terms of quadratic forms \pcite{colbrook2021rigorous,colbrook2023residual}:
\begin{equation}
  \begin{aligned}
  &\widehat{\tau}_{M,N}(\lambda)^2\\
  &\quad=
  \min_{\substack{\mathbf{a}\in\mathbb{C}^{d_N}\\\mathbf{a}\neq 0}}
  \frac{
    \mathbf{a}^\ast
    \left(
      \begin{aligned}
      &\mathbf{G}_{YY}^{(M,N)}
      -\lambda (\mathbf{G}_{XY}^{(M,N)})^\ast\\
      &\quad-\overline{\lambda} \mathbf{G}_{XY}^{(M,N)}
      +|\lambda|^2 \mathbf{G}_{XX}^{(M,N)}
      \end{aligned}
    \right)\mathbf{a}
  }{
      \mathbf{a}^\ast \mathbf{G}_{XX}^{(M,N)} \mathbf{a}
    }.
  \end{aligned}
  \label{eq:appendix_empirical_tau}
\end{equation}
With the point evaluations fixed, the minimum is independent of the choice of basis.

\subsubsection*{Assumptions for the data limit}
\begin{itemize}
\item[(i)] The snapshot pairs \((\mathbf{x}^{(i)},\mathbf{y}^{(i)})\) are generated from independent and identically distributed samples \(\mathbf{x}^{(i)}\sim\mu\); or
\item[(ii)] \(F\) is \(\mu\)-measure-preserving and ergodic, and the snapshot pairs are drawn from a single trajectory \(\mathbf{x}^{(i)}=F^{i}(\mathbf{x}^{(0)})\), with the trajectory averages of the integrable Gram observables converging to their integrals for almost every initial condition \(\mathbf{x}^{(0)}\).
\end{itemize}
\begin{proposition}[Convergence with respect to the number of snapshot pairs ($M\to\infty$)]
\label{prop:conv-data}
Let \(\mu\) be a probability measure, fix \(N\), \(V_N\), and a basis of \(V_N\) independently of the evaluation samples, and assume sampling condition (i) or (ii) above.
Then, on a common set of full measure, \(\mathbf{G}_{XX}^{(M,N)}\) is positive definite for all sufficiently large \(M\).
Under (i), this set is an event of probability one for the sample sequence; under (ii), it contains \(\mu\)-almost every initial condition.
On the same set, the following holds for every compact set \(\Lambda\subset\mathbb{C}\).
\[
\sup_{\lambda\in\Lambda}
\left|
\widehat{\tau}_{M,N}(\lambda)-\tau_N(\lambda)
\right|
\to 0
\qquad
(M\to\infty).
\]
\end{proposition}
\begin{proof}[Proof of Proposition~\ref{prop:conv-data}]
Since \(b_j,\mathcal{K}b_j\in\mathcal{H}\), the products below are integrable by the Cauchy--Schwarz inequality.
By the strong law of large numbers under condition (i) and Birkhoff's ergodic theorem under condition (ii), the following entrywise convergence holds on a set of full measure common to all \(j,k\) \pcite{colbrook2021rigorous}:
\[
\begin{aligned}
\lim_{M\to\infty}[\mathbf{G}_{XX}^{(M,N)}]_{jk}&=\int_\Omega b_k(\mathbf{x})\overline{b_j(\mathbf{x})}\,\mathrm{d}\mu(\mathbf{x}),\\
\lim_{M\to\infty}[\mathbf{G}_{XY}^{(M,N)}]_{jk}&=\int_\Omega b_k(F(\mathbf{x}))\overline{b_j(\mathbf{x})}\,\mathrm{d}\mu(\mathbf{x}),\\
\lim_{M\to\infty}[\mathbf{G}_{YY}^{(M,N)}]_{jk}&=\int_\Omega b_k(F(\mathbf{x}))\overline{b_j(F(\mathbf{x}))}\,\mathrm{d}\mu(\mathbf{x}).
\end{aligned}
\]
Since only finitely many matrix entries are involved, the three matrices also converge in norm.
By linear independence of the basis, the limit of \(\mathbf{G}_{XX}^{(M,N)}\) is positive definite, and \(\mathbf{G}_{XX}^{(M,N)}\) is also positive definite for all sufficiently large \(M\).
The quotient in Eq.~\eqref{eq:appendix_empirical_tau} is homogeneous in \(\mathbf{a}\), so the minimization can be restricted to the coefficient unit sphere, and the minimum of the population quotient is \(\tau_N(\lambda)^2\).
Since the limiting Gram matrix is positive definite, the empirical denominators are bounded below by a common positive constant on the coefficient unit sphere for all sufficiently large \(M\).
The convergence of the Gram and cross-correlation matrices then gives uniform convergence of the Rayleigh quotients on the product of this sphere and any compact set \(\Lambda\subset\mathbb{C}\), and hence of their minima.
Finally, the inequality \(|\sqrt{u}-\sqrt{v}|\le\sqrt{|u-v|}\) for nonnegative \(u,v\) gives uniform convergence of \(\widehat{\tau}_{M,N}\) to \(\tau_N\) on \(\Lambda\), including at points of zero residual.
\end{proof}

\subsubsection*{Assumption for the latent-dimension limit}
Density results for feedforward networks are classical \pcite{hornik1991,leshno1993nonpolynomial}, and for real-valued ReLU networks with two hidden layers, density in \(L^2(\mathbb{R}^{d_x})\) with Lebesgue measure is known \pcite{wang2022unbounded}.
Here, Assumption (A3) plays the role of the core condition in \tcite[Appendix~B]{colbrook2021rigorous}.

\begin{proof}[Proof of Proposition~\ref{prop:conv-latent}]
Fix \(g\in\mathcal{H}\) and \(\delta>0\).
Because \(E_N(g)\to0\), each sufficiently large \(N\) admits \(v_N\in V_N\) satisfying \(\|g-v_N\|_{\mathcal{H}}<\delta/(1+\|\mathcal{K}\|)\), and boundedness gives
\[
\|g-v_N\|_\mathcal{H}+\|\mathcal{K}g-\mathcal{K}v_N\|_\mathcal{H}
\leq
(1+\|\mathcal{K}\|)\|g-v_N\|_\mathcal{H}
<\delta.
\]
\noindent Next, fix \(\lambda\in\mathbb{C}\) and \(\eta>0\), and choose \(g\in\mathcal{H}\) such that \(\|g\|_{\mathcal{H}}=1\) and \(\operatorname{res}(\lambda,g)<\tau(\lambda)+\eta\) \pcite[Appendix~B]{colbrook2021rigorous}.
By Assumption (A3), we can choose \(v_N\in V_N\) such that \(v_N\to g\), and boundedness gives
\[
\|(\mathcal{K}-\lambda I)(v_N-g)\|_\mathcal{H}
\le(\|\mathcal{K}\|+|\lambda|)\|v_N-g\|_\mathcal{H}
\to0.
\]
Moreover, since \(\|v_N\|_{\mathcal{H}}\to1\), for sufficiently large \(N\) we have \(v_N\ne0\) and \(\operatorname{res}(\lambda,v_N)\to\operatorname{res}(\lambda,g)\).
From the inclusion \(V_N\subset\mathcal{H}\) and \(\tau_N(\lambda)\le\operatorname{res}(\lambda,v_N)\), we obtain
\[
\begin{aligned}
\tau(\lambda)
&\le\liminf_{N\to\infty}\tau_N(\lambda)
\le\limsup_{N\to\infty}\tau_N(\lambda)\\
&\le\operatorname{res}(\lambda,g)
<\tau(\lambda)+\eta.
\end{aligned}
\]
Since \(\eta>0\) was arbitrary, we obtain \(\tau_N(\lambda)\to\tau(\lambda)\).
Therefore, if \(\tau(\lambda)<\epsilon\), then \(\lambda\in\sigma_{N,\epsilon,\mathrm{ap}}(\mathcal{K})\) for all sufficiently large \(N\).
\end{proof}

\subsubsection*{Approximate point pseudospectra}
Proposition~\ref{prop:conv-latent} and the proposition below concern \(\sigma_{N,\epsilon,\mathrm{ap}}(\mathcal{K})\), \(\sigma_{\epsilon,\mathrm{ap}}(\mathcal{K})\), and \(\sigma_{\mathrm{ap}}(\mathcal{K})\), defined in Sect.~\ref{sec:residuals-spectral-diagnostics} \pcite[Appendix~B]{colbrook2021rigorous}.
For a general nonnormal operator, treating points in the residual spectrum may require a construction using the corresponding adjoint operator.
Here, \(\sigma_{\mathrm{ap}}(\mathcal{K})\) coincides with the standard approximate point spectrum of a bounded operator on a nonzero Hilbert space, the strict sublevel sets are open, and \(\sigma_{\mathrm{ap}}(\mathcal{K})\) is closed.

\begin{proposition}[Zero-residual characterization ($\epsilon\downarrow0$)]
\label{prop:conv-epsilon}
Under Assumptions (A1)--(A4), on a common set of full measure for all \(N\), \(\lim_{N\to\infty}\lim_{M\to\infty}\widehat{\tau}_{M,N}(\lambda)=\tau(\lambda)\) for every \(\lambda\in\mathbb{C}\).
For each \(\lambda\in\mathbb{C}\),
\[
\begin{aligned}
\lambda\in\sigma_{\mathrm{ap}}(\mathcal{K})
&\iff\tau(\lambda)=0\\
&\iff\bigl[\forall\epsilon>0,\ \lambda\in\sigma_{\epsilon,\mathrm{ap}}(\mathcal{K})\bigr].
\end{aligned}
\]
\end{proposition}
\begin{proof}[Proof of Proposition~\ref{prop:conv-epsilon}]
Under Assumption (A4), the argument used in the proof of Proposition~\ref{prop:conv-data} to deduce uniform convergence from matrix convergence shows that, for each \(N\), there is a set of full measure on which the convergence \(\widehat{\tau}_{M,N}\to\tau_N\) is uniform on every compact subset of \(\mathbb{C}\).
On the countable intersection of these sets over \(N\), Proposition~\ref{prop:conv-latent} gives \(\lim_{N\to\infty}\lim_{M\to\infty}\widehat{\tau}_{M,N}(\lambda)=\tau(\lambda)\) for every \(\lambda\in\mathbb{C}\).
From Eq.~\eqref{eq:spectral-sets} and \(\tau(\lambda)\ge0\), we obtain
\[
\bigcap_{\epsilon>0}\sigma_{\epsilon,\mathrm{ap}}(\mathcal{K})
=\{\lambda\in\mathbb{C}:\tau(\lambda)=0\}
=\sigma_{\mathrm{ap}}(\mathcal{K}).
\]
If \(\tau(\lambda)=0\), then for every \(\epsilon>0\), all sufficiently large \(N\) satisfy \(\tau_N(\lambda)<\epsilon/2\), and for each such \(N\), all sufficiently large \(M\) satisfy \(\widehat{\tau}_{M,N}(\lambda)<\epsilon\).
Conversely, if \(\tau(\lambda)>0\), choose \(0<\epsilon<\tau(\lambda)\); the same successive limits then give \(\widehat{\tau}_{M,N}(\lambda)>\epsilon\) for all sufficiently large \(N\) and, for each such \(N\), all sufficiently large \(M\).
\end{proof}

\subsection*{Reconstruction and exact latent collisions}
The bound applies to each trained model with its own \(e_{\mathrm{rec}}\).

\begin{lemma}[A pointwise reconstruction bound limits exact latent collisions]
\label{lem:reconstruction-excludes-collapse}
Assume the following for every sampled observation vector \(\mathbf{x}\) under consideration.
\begin{equation}
  \|\mathbf{x}-\psi_{\theta_{\mathrm{d}}}(\phi_{\theta_{\mathrm{e}}}(\mathbf{x}))\|_2
  \leq
  e_{\mathrm{rec}}.
  \label{eq:appendix_reconstruction_bound}
\end{equation}
If $\mathbf{x}^{(i)}$ and $\mathbf{x}^{(j)}$ satisfy $\phi_{\theta_{\mathrm{e}}}(\mathbf{x}^{(i)})=\phi_{\theta_{\mathrm{e}}}(\mathbf{x}^{(j)})$, then
\begin{equation}
  \|\mathbf{x}^{(i)}-\mathbf{x}^{(j)}\|_2
  \leq
  2e_{\mathrm{rec}}.
  \label{eq:appendix_collapse_bound}
\end{equation}
\end{lemma}
\begin{proof}[Proof of Lemma~\ref{lem:reconstruction-excludes-collapse}]
If $\phi_{\theta_{\mathrm{e}}}(\mathbf{x}^{(i)})=\phi_{\theta_{\mathrm{e}}}(\mathbf{x}^{(j)})$, the decoder outputs coincide.
Therefore
\[
\begin{aligned}
\|\mathbf{x}^{(i)}-\mathbf{x}^{(j)}\|_2
&\leq
\|\mathbf{x}^{(i)}-\psi_{\theta_{\mathrm{d}}}(\phi_{\theta_{\mathrm{e}}}(\mathbf{x}^{(i)}))\|_2
\\
&\quad+
\|\mathbf{x}^{(j)}-\psi_{\theta_{\mathrm{d}}}(\phi_{\theta_{\mathrm{e}}}(\mathbf{x}^{(j)}))\|_2
\\
&\leq
2e_{\mathrm{rec}}.
\end{aligned}
\]
\end{proof}

\section{Training and baseline settings}
\label{app:settings}
\subsection*{Common preparation}
\subsubsection*{Initial conditions}
The initial conditions are based on \pcite{brunton2017havok} for R\"ossler, Duffing, and Mackey--Glass, \pcite{vlachas2020backpropagation} for Lorenz-63, \pcite{tsuyuki2022nonlinear} for Lorenz-96, and \pcite{otto2019linearly,su2026koopgen} for Kuramoto--Sivashinsky.
For R\"ossler, Lorenz-63, Duffing, and Lorenz-96, we add independent zero-mean Gaussian perturbations with standard deviation \(1\) to each component of the reference states \((1,1,1)\), \((1,1,1)\), \((0,0.01)\), and the state with all components equal to \(F_{\mathrm{L96}}\), respectively.
For Mackey--Glass, the initial history is the constant \(\max(0.5+\xi,10^{-6})\), where \(\xi\) is a standard normal random variable.
For Kuramoto--Sivashinsky, we set \(\hat u_k=0.1(\xi_k+i\eta_k)\) for Fourier modes \(k=1,2,3\), use their conjugates for negative modes, and set all other coefficients to zero, where \(\xi_k\) and \(\eta_k\) are independent standard normal random variables.
The step counts in the Steps column of Table~\ref{tab:system-settings} include the burn-in steps, which are discarded before the data are split, and random seeds vary the initial-condition perturbations and minibatch sampling.

\subsubsection*{Computational environment}
All experiments ran on NVIDIA RTX 6000 Ada graphics processing units (GPUs).

\subsection*{Koopman autoencoder settings}
\begin{table*}[t]
\centering
\caption{Training parameters for the Koopman autoencoders.}
\label{tab:opt-settings}
\begin{tabular}{@{}ll@{}}
\toprule
Parameter & Value \\
\midrule
Hidden layers per network & \(2\) \\
Hidden width & \(N\) \\
Activation & ReLU \\
Optimizer & AdamW \\
Learning rate & \(10^{-4}\) \\
Batch size & \(2{,}048\) \\
\bottomrule
\end{tabular}
\end{table*}
\subsubsection*{Training protocol}
We train for 200 epochs using the settings in Table~\ref{tab:opt-settings}, with \(20\) warm-up epochs followed by \(180\) epochs of cosine learning-rate decay and weight decay of \(10^{-2}\).
The encoder and decoder are initialized with autoencoder weights pretrained using a reconstruction loss and a learning rate of \(3\times10^{-4}\) for \(300{,}000\) updates on a separate trajectory of the same system at the same latent dimension.
Two-stage training allows at most \(J_1=J_2=250\) updates per stage in each epoch, and joint training allows at most \(250\) updates per epoch.
In two-stage training, each stage stops when the relative decrease between the mean losses in the two halves of a \(150\)-update window, expressed per \(100\) updates, is at most \(1\%\).
For each operator estimate, we draw \(2{,}048\) pairs for operator estimation and \(2{,}048\) pairs for evaluating the evolution loss from two disjoint pools formed by assigning successive temporal blocks of \(2\) LT alternately to the two pools.
The regularization parameter is \(\varepsilon=10^{-6}\|\mathbf{Z}_X\|_F^2/N\), and the auxiliary comparisons use the operator-norm penalty weight \(w_{\mathrm{op}}=0.01\).
\subsubsection*{Operators for rollouts and validation}
Both models estimate \(\widehat{\mathbf{A}}_{\mathrm{tr}}\) from all \(M_{\mathrm{tr}}\) encoded training pairs \(\mathbf{Z}_{X,\mathrm{tr}},\mathbf{Z}_{Y,\mathrm{tr}}\in\mathbb{R}^{N\times M_{\mathrm{tr}}}\) using Eq.~\eqref{eq:latent_operator} with regularization parameter \(\varepsilon_{\mathrm{tr}}=10^{-6}\|\mathbf{Z}_{X,\mathrm{tr}}\|_F^2/N\).
The spectral-residual model uses \(\widehat{\mathbf{A}}_{\mathrm{tr}}\) for rollouts.
For latent prediction, rollouts with \(\widehat{\mathbf{A}}_{\mathrm{tr}}\) diverged, so we use a rank-constrained operator.
In the SVD \(\mathbf{Z}_{X,\mathrm{tr}}=\mathbf{U}_{\mathrm{tr}}\boldsymbol{\Sigma}_{\mathrm{tr}}\mathbf{V}_{\mathrm{tr}}^\top\), with singular values \(\sigma_j\) in nonincreasing order, we choose the smallest rank \(r\) that retains at least \(99.9\%\) of the sum of the squared singular values.
With \(\mathbf{U}_{\mathrm{tr},r}\) and \(\mathbf{V}_{\mathrm{tr},r}\) denoting the first \(r\) singular-vector columns and \(\boldsymbol{\Sigma}_{\mathrm{tr},r}\) the diagonal matrix of the retained singular values, let \(\mathbf{Q}_r=\mathbf{Z}_{Y,\mathrm{tr}}\mathbf{V}_{\mathrm{tr},r}\boldsymbol{\Sigma}_{\mathrm{tr},r}(\boldsymbol{\Sigma}_{\mathrm{tr},r}^2+\varepsilon_{\mathrm{tr}}\mathbf{I}_r)^{-1}\) and \(\widehat{\mathbf{A}}_{\mathrm{tr},r}=\mathbf{Q}_r \mathbf{U}_{\mathrm{tr},r}^\top \mathbf{Q}_r \mathbf{Q}_r^\dagger\).
Here \(\mathbf{I}_r\) is the \(r\times r\) identity matrix and \(\dagger\) denotes the Moore--Penrose pseudoinverse.
Validation uses this rank-constrained construction for both models, with \(r\) equal to the number of singular directions satisfying \(\sigma_j^2\ge\varepsilon_{\mathrm{tr}}\).
The eigenvalues overlaid on the residual contours are those of \(\widehat{\mathbf{A}}_{\mathrm{tr}}\) for both models.
\subsubsection*{Residual contours}
In the SVD described in Sect.~\ref{sec:metrics-analysis}, we retain the singular values satisfying \(\sigma_j>10^{-10}\sigma_1\), where \(\sigma_1\) is the largest singular value.
We evaluate the residual contours on a \(150\times150\) grid over \([-1.45,1.45]^2\), with the color range determined separately for each figure from the combined values of its eight panels.
\subsection*{Baseline settings}
We select the Neural ODE and Consistent KAE models with the lowest VRMSE on the validation interval.
\subsubsection*{Latent-state Neural ODE}
The encoder, decoder, and latent vector field of the Neural ODE \pcite{chen2018neuralode} have input and output dimensions \(d_x,N\), \(N,d_x\), and \(N,N\), respectively, with one hidden layer each in the encoder and decoder and two hidden layers in the latent vector field.
All hidden layers have width \(128\) and \(\tanh\) activations, all output layers are linear, and we integrate the latent ODE with dopri5.
We train with Adam at a constant learning rate of \(10^{-3}\) for \(300\) epochs with batch size \(8{,}192\), using every training starting point once per epoch in minibatches.
The training loss is the mean squared error between observations and decoded values over \(n\) steps of autonomous evolution from an encoded starting point.
For training and validation, we use \(n=q\).

\subsubsection*{Kernel DMD with random Fourier features}
Kernel DMD uses RFF \pcite{rahimi2007random,williams2015data,degennaro2019scalable} with the RBF kernel \(\exp(-\gamma\|\mathbf{x}-\mathbf{x}'\|_2^2)\) (\(\gamma=1.0\)) and fits on all training pairs.
The representation size is \(N=2D\).
We sample the frequencies \(\mathbf{W}_{\mathrm{RFF}}\in\mathbb{R}^{D\times d_x}\) with normally distributed entries of mean zero and variance \(2\gamma\), and the feature map is \(\phi_{\mathrm{RFF}}(\mathbf{x})=D^{-1/2}[\cos(\mathbf{W}_{\mathrm{RFF}}\mathbf{x})^\top,\sin(\mathbf{W}_{\mathrm{RFF}}\mathbf{x})^\top]^\top\in\mathbb{R}^{2D}\).
Using the same training snapshots, we learn a linear readout and use it to reconstruct the input state from the observables advanced by repeated application of the EDMD operator.
For the reported rollouts, we retain the sampled features and estimate the EDMD operator at the smallest rank retaining \(99.9\%\) of the sum of the squared singular values of the training feature matrix.

\subsubsection*{Consistent KAE}
Consistent KAE \pcite{azencot2020forecasting} combines an MLP encoder and decoder connecting observation dimension \(d_x\) and latent dimension \(N\) with \(N\times N\) forward and backward evolution matrices.
Each network has two hidden layers of width \(N\) with \(\tanh\) activations, and the decoder output layer is linear, as in the original paper.
Forward and backward evolution each use \(4\) steps, with loss weights of \(1\) for forward evolution, \(1\) for reconstruction, \(0.01\) for backward evolution, and \(0.001\) for consistency.
We train with AdamW at an initial learning rate of \(10^{-3}\) and weight decay \(0\) (the default of the original implementation) for \(100\) epochs with batch size \(1{,}024\), multiplying the learning rate by \(0.2\) at epochs \(5\), \(33\), \(67\), and \(83\).
Dimension-sweep evaluation projects the learned forward operator onto the training latent subspace using the same \(99.9\%\) energy threshold.

\subsubsection*{Echo state network}
The \(\tanh\) reservoir \pcite{pathak2018model} uses a spectral radius of \(0.9\), a leaking rate of \(0.3\), input scaling of \(1.0\), input and reservoir connectivity of \(0.1\), and a regularization coefficient of \(10^{-6}\).
After discarding the initial \(100\) training steps, we fit the ESN readout by closed-form ridge regression to predict the one-step-ahead input vector from the reservoir state.
For prediction, we reset the reservoir state to zero and feed it up to \(q\) preceding observation vectors in sequence, followed by the starting observation vector.
We then advance it autoregressively.

%% file: declarations.tex
\section*{Statements and Declarations}
\subsection*{Funding}
This work was supported in part by JSPS KAKENHI Grant Numbers JP24K20864 and JP26H02498 from the Japan Society for the Promotion of Science and by JST Moonshot R\&D Program Grant Number JPMJMS263K from the Japan Science and Technology Agency.
\subsection*{Competing Interests}
The authors have no relevant financial or non-financial interests to disclose.
\subsection*{Author Contributions}
I.S., Y.M., and Y.K. conceived the study and developed the method.
Y.K. supervised the project.
I.S. performed the theoretical analysis, conducted the experiments, and wrote the first draft.
I.S. and Y.M. implemented the software.
All authors reviewed and edited the manuscript, and read and approved the final manuscript.
\subsection*{Data Availability}
The code and configuration files used to generate the data and the numerical results included in this paper are available at \url{https://github.com/uosaka-mlsyslab/koopman-autoencoder-spectral-reliability}.

%% file: supplement_body.tex
This supplement provides additional figures for the two analyses in Sect.~5 of the main text.
Sections~\ref{supp:lorenz63}--\ref{supp:ks} pair spectral plots for eight model variants with predictions at latent dimension \(16N_0\) for the five systems other than R\"ossler.
The supplemental figures follow the procedures for residual contours and predictions described in the main text.
Time is expressed in Lyapunov time (LT), the inverse of the largest Lyapunov exponent of each system.

\clearpage
\section{Lorenz-63 system}
\label{supp:lorenz63}
Figures~\ref{fig:supp-pseudospectra-lorenz63} and~\ref{fig:supp-decoded-lorenz63} show the spectral plots and predictions for Lorenz-63.
In the prediction figure, the spectral-residual model with two-stage training captures the overall shape of the reference attractor at the displayed prediction horizons up to 1.0 LT, and two lobes smaller than those of the reference attractor can be distinguished (Sect.~5 of the main text).

\begin{figure}[htbp]
\centering
\includegraphics[scale=1]{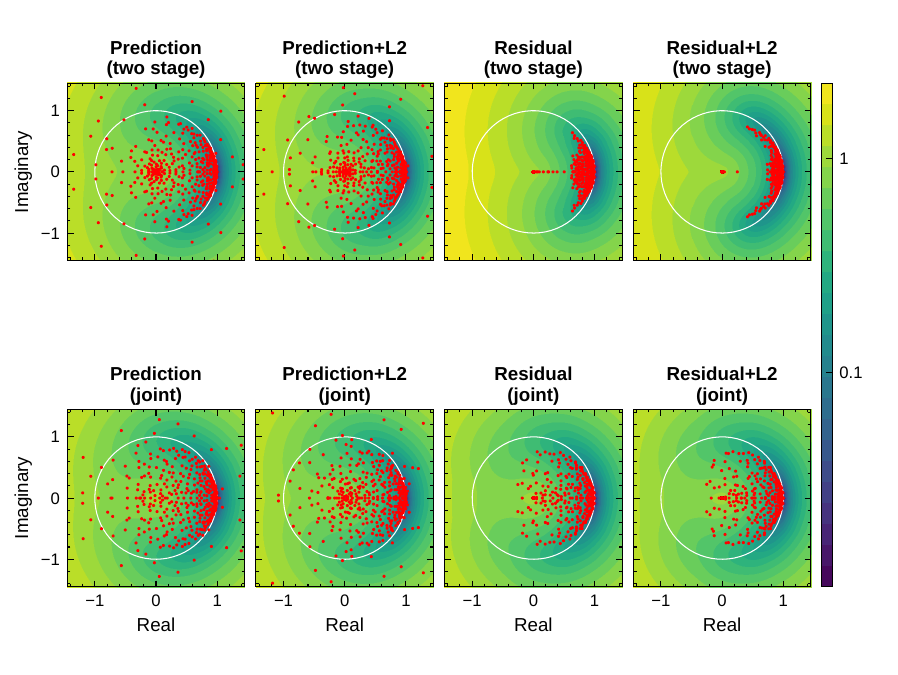}
\caption{Minimal residuals \(\widehat{\tau}^{\mathrm{num}}_{M,N}\) and eigenvalues for eight model variants in Lorenz-63 at latent dimension \(16N_0\).
Colors show, on a logarithmic scale, contours of the minimal residual computed over the retained numerical subspace, with the same color range across all panels.
White curves mark the unit circle, and red points mark the estimated Koopman eigenvalues within the plotted range.
The top row shows two-stage training, and the bottom row shows joint training.
From left to right, the columns show Prediction (the latent-prediction model), Prediction+L2, Residual (the spectral-residual model), and Residual+L2; +L2 adds the operator-norm penalty}
\label{fig:supp-pseudospectra-lorenz63}
\end{figure}

\begin{figure}[htbp]
\centering
\includegraphics[scale=1]{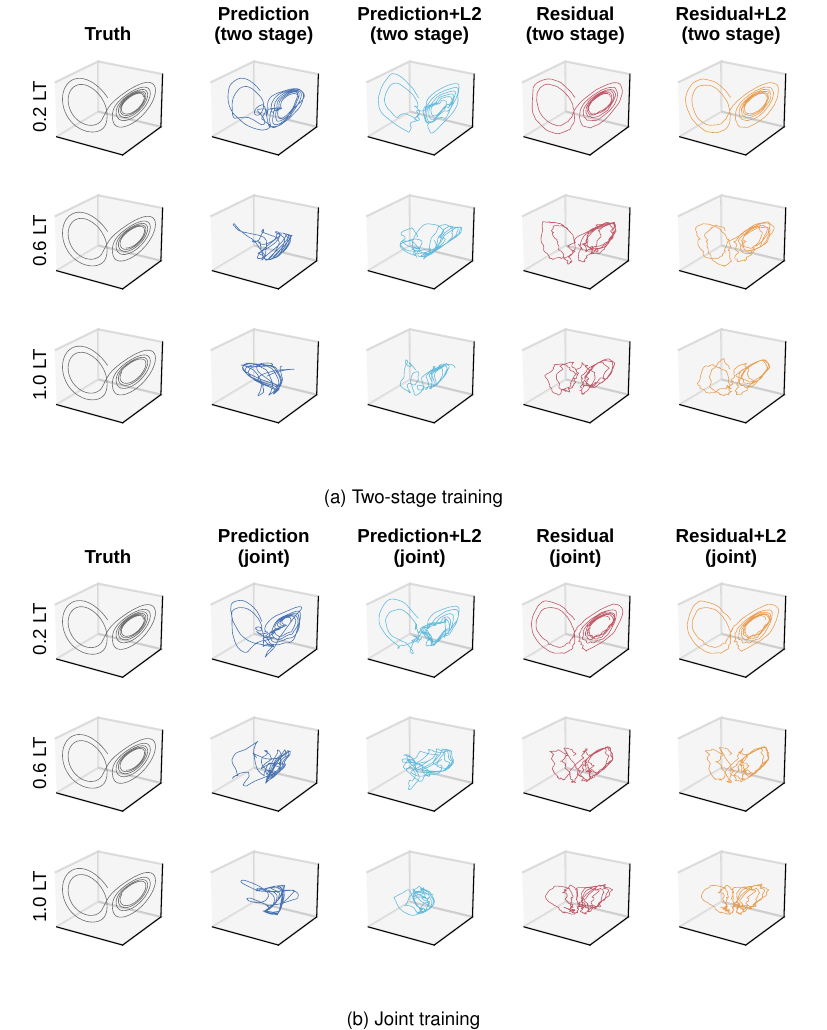}
\caption{Predictions for Lorenz-63 at latent dimension \(16N_0\) (top: two-stage training; bottom: joint training).
The predictions are plotted in \((x,y,z)\) coordinates.
From left to right, the columns show the reference trajectory, Prediction (the latent-prediction model), Prediction+L2, Residual (the spectral-residual model), and Residual+L2; +L2 adds the operator-norm penalty.
Each row corresponds to one prediction horizon (0.2, 0.6, or 1.0 LT), and predictions are made that far ahead from different observation vectors in the reference trajectory.
The prediction target times fall within the first 5 LT of the test trajectory.
Axis limits follow the reference trajectory and are shared across all panels}
\label{fig:supp-decoded-lorenz63}
\end{figure}

\clearpage
\section{Duffing system}
\label{supp:duffing}
Figures~\ref{fig:supp-pseudospectra-duffing} and~\ref{fig:supp-decoded-duffing} show the spectral plots and predictions for Duffing.
In the prediction figure, the spectral-residual model with two-stage training captures the overall shape of the reference attractor at the displayed prediction horizons up to 1.0 LT, and predictions from the same model with joint training have a noticeably smaller extent from 0.6 LT onward (Sect.~5 of the main text).

\begin{figure}[htbp]
\centering
\includegraphics[scale=1]{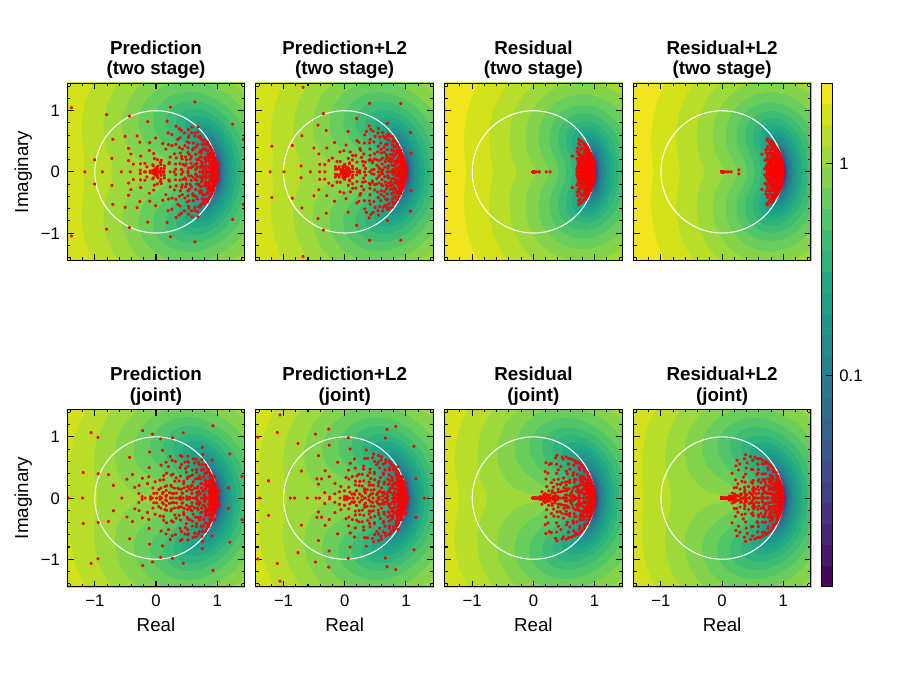}
\caption{Minimal residuals \(\widehat{\tau}^{\mathrm{num}}_{M,N}\) and eigenvalues for eight model variants in Duffing at latent dimension \(16N_0\).
Colors show, on a logarithmic scale, contours of the minimal residual computed over the retained numerical subspace, with the same color range across all panels.
White curves mark the unit circle, and red points mark the estimated Koopman eigenvalues within the plotted range.
The top row shows two-stage training, and the bottom row shows joint training.
From left to right, the columns show Prediction (the latent-prediction model), Prediction+L2, Residual (the spectral-residual model), and Residual+L2; +L2 adds the operator-norm penalty}
\label{fig:supp-pseudospectra-duffing}
\end{figure}

\begin{figure}[htbp]
\centering
\includegraphics[scale=1]{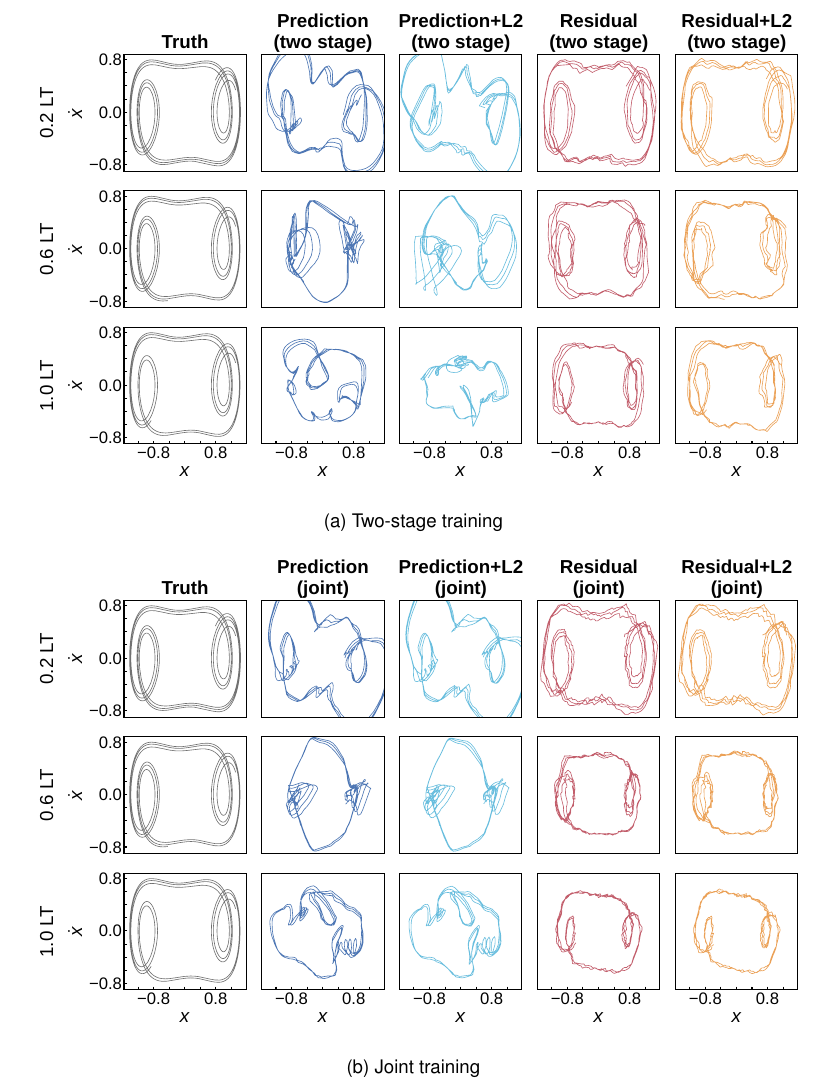}
\caption{Predictions for Duffing at latent dimension \(16N_0\) (top: two-stage training; bottom: joint training).
The predictions are plotted in the \((x,\dot{x})\) plane.
From left to right, the columns show the reference trajectory, Prediction (the latent-prediction model), Prediction+L2, Residual (the spectral-residual model), and Residual+L2; +L2 adds the operator-norm penalty.
Each row corresponds to one prediction horizon (0.2, 0.6, or 1.0 LT), and predictions are made that far ahead from different observation vectors in the reference trajectory.
The prediction target times fall within the first 5 LT of the test trajectory.
Axis limits follow the reference trajectory and are shared across all panels}
\label{fig:supp-decoded-duffing}
\end{figure}

\clearpage
\section{Mackey--Glass system}
\label{supp:mackeyglass}
Figures~\ref{fig:supp-pseudospectra-mackeyglass} and~\ref{fig:supp-decoded-mackeyglass} show the spectral plots and predictions for Mackey--Glass.
In the spectral-residual panels for two-stage training, low-residual regions form a band along the unit circle.

\begin{figure}[htbp]
\centering
\includegraphics[scale=1]{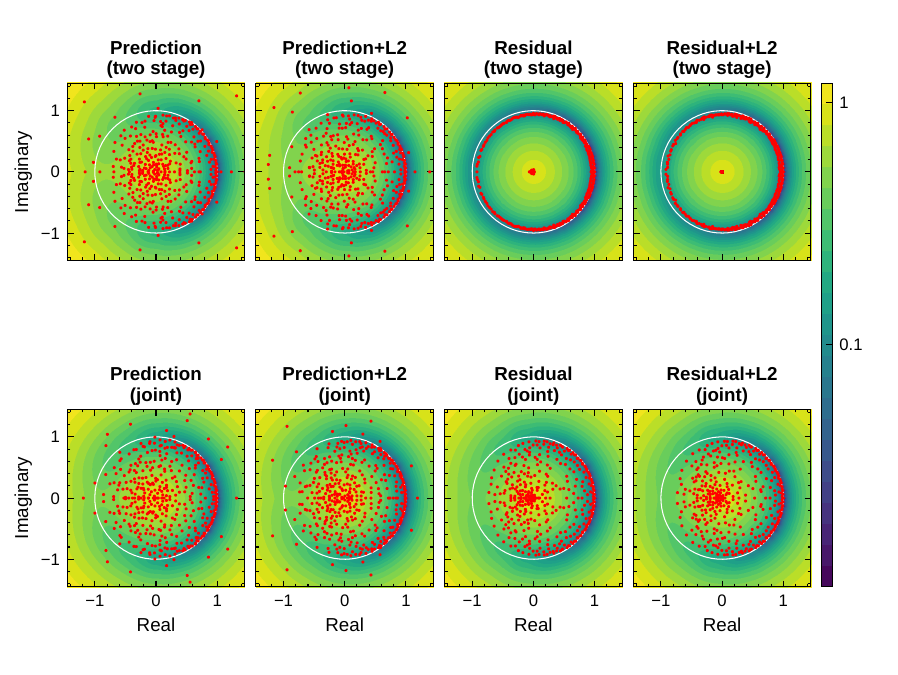}
\caption{Minimal residuals \(\widehat{\tau}^{\mathrm{num}}_{M,N}\) and eigenvalues for eight model variants in Mackey--Glass at latent dimension \(16N_0\).
Colors show, on a logarithmic scale, contours of the minimal residual computed over the retained numerical subspace, with the same color range across all panels.
White curves mark the unit circle, and red points mark the estimated Koopman eigenvalues within the plotted range.
The top row shows two-stage training, and the bottom row shows joint training.
From left to right, the columns show Prediction (the latent-prediction model), Prediction+L2, Residual (the spectral-residual model), and Residual+L2; +L2 adds the operator-norm penalty}
\label{fig:supp-pseudospectra-mackeyglass}
\end{figure}

\begin{figure}[htbp]
\centering
\includegraphics[scale=1]{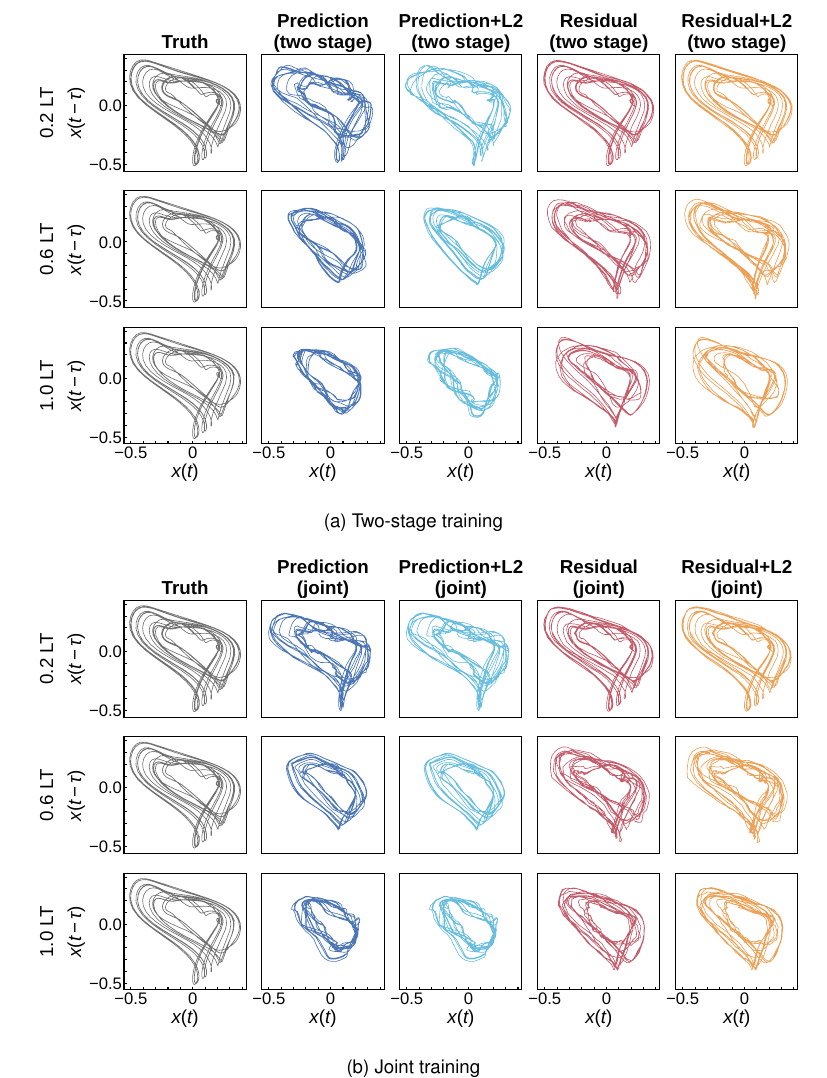}
\caption{Predictions for Mackey--Glass at latent dimension \(16N_0\) (top: two-stage training; bottom: joint training).
The predictions are plotted in the delay plane \((x(t),x(t-\tau))\), with \(\tau=17\) in the time units of the governing equation.
From left to right, the columns show the reference trajectory, Prediction (the latent-prediction model), Prediction+L2, Residual (the spectral-residual model), and Residual+L2; +L2 adds the operator-norm penalty.
Each row corresponds to one prediction horizon (0.2, 0.6, or 1.0 LT), and predictions are made that far ahead from different observation vectors in the reference trajectory.
The prediction target times fall within the first 5 LT of the test trajectory.
Axis limits follow the reference trajectory and are shared across all panels}
\label{fig:supp-decoded-mackeyglass}
\end{figure}

\clearpage
\section{Lorenz-96 system}
\label{supp:lorenz96}
Figures~\ref{fig:supp-pseudospectra-lorenz96} and~\ref{fig:supp-decoded-lorenz96} show the spectral plots and predictions for Lorenz-96.
In the spectral plot, most estimated eigenvalues of the latent-prediction model in the Prediction column cluster in low-residual regions near \(\lambda=1\) under both two-stage and joint training (Sect.~5 of the main text), and more of them are scattered into high-residual regions under two-stage training.
In the prediction figure, the striped structure of the field is distorted at 0.6 LT for all eight model variants.
For Lorenz-96, increasing the latent dimension from \(N_0\) to \(2N_0\) increased the mean valid prediction time (VPT) of the spectral-residual model under two-stage training from 0.198 to 0.211 LT, while median Variance-Scaled Root Mean Squared Error (VRMSE) increased from 0.9959 to 0.9980 (Sect.~4 of the main text).

\begin{figure}[htbp]
\centering
\includegraphics[scale=1]{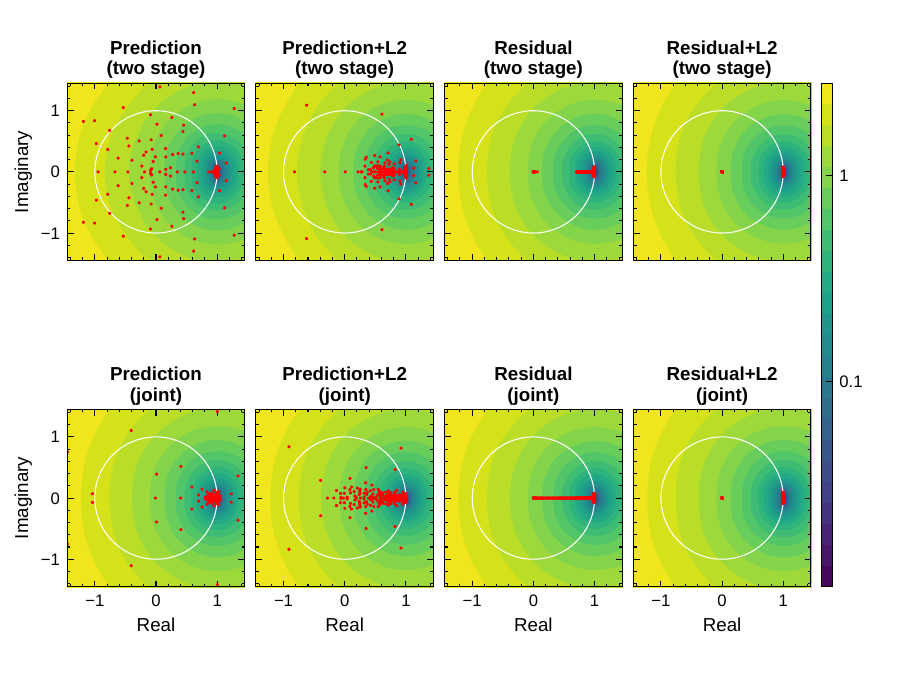}
\caption{Minimal residuals \(\widehat{\tau}^{\mathrm{num}}_{M,N}\) and eigenvalues for eight model variants in Lorenz-96 at latent dimension \(16N_0\).
Colors show, on a logarithmic scale, contours of the minimal residual computed over the retained numerical subspace, with the same color range across all panels.
White curves mark the unit circle, and red points mark the estimated Koopman eigenvalues within the plotted range.
The top row shows two-stage training, and the bottom row shows joint training.
From left to right, the columns show Prediction (the latent-prediction model), Prediction+L2, Residual (the spectral-residual model), and Residual+L2; +L2 adds the operator-norm penalty}
\label{fig:supp-pseudospectra-lorenz96}
\end{figure}

\begin{figure}[htbp]
\centering
\includegraphics[scale=1]{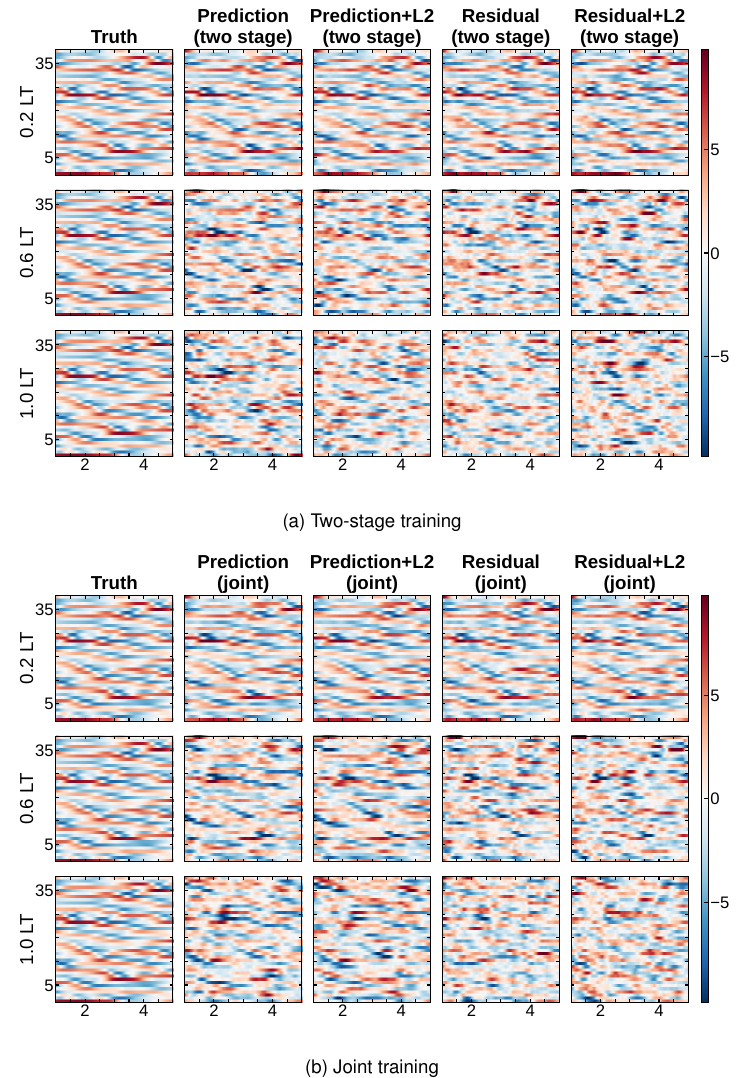}
\caption{Predictions for Lorenz-96 at latent dimension \(16N_0\) (top: two-stage training; bottom: joint training).
From left to right, the columns show the reference field, Prediction (the latent-prediction model), Prediction+L2, Residual (the spectral-residual model), and Residual+L2; +L2 adds the operator-norm penalty.
Each row corresponds to one prediction horizon (0.2, 0.6, or 1.0 LT), and predictions are made that far ahead from different observation vectors in the reference trajectory.
Horizontal axis: prediction target time in LT; vertical axis: component index 0--39; color: component value, with the same color range across all panels.
Component indices 0--39 correspond, in order, to \(x_1,\ldots,x_{40}\) in the main text}
\label{fig:supp-decoded-lorenz96}
\end{figure}

\clearpage
\section{Kuramoto--Sivashinsky system}
\label{supp:ks}
Figures~\ref{fig:supp-pseudospectra-ks} and~\ref{fig:supp-decoded-ks} show the spectral plots and predictions for Kuramoto--Sivashinsky.
In the prediction figure, the striped structure of the field is distorted at 0.6 LT for all eight model variants.

\begin{figure}[htbp]
\centering
\includegraphics[scale=1]{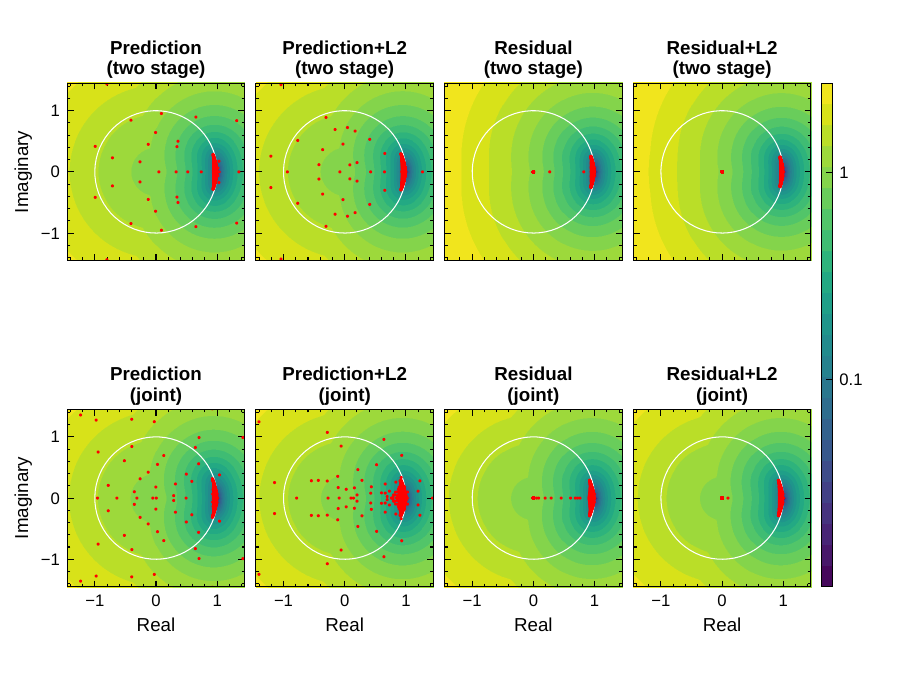}
\caption{Minimal residuals \(\widehat{\tau}^{\mathrm{num}}_{M,N}\) and eigenvalues for eight model variants in Kuramoto--Sivashinsky at latent dimension \(16N_0\).
Colors show, on a logarithmic scale, contours of the minimal residual computed over the retained numerical subspace, with the same color range across all panels.
White curves mark the unit circle, and red points mark the estimated Koopman eigenvalues within the plotted range.
The top row shows two-stage training, and the bottom row shows joint training.
From left to right, the columns show Prediction (the latent-prediction model), Prediction+L2, Residual (the spectral-residual model), and Residual+L2; +L2 adds the operator-norm penalty}
\label{fig:supp-pseudospectra-ks}
\end{figure}

\begin{figure}[htbp]
\centering
\includegraphics[scale=1]{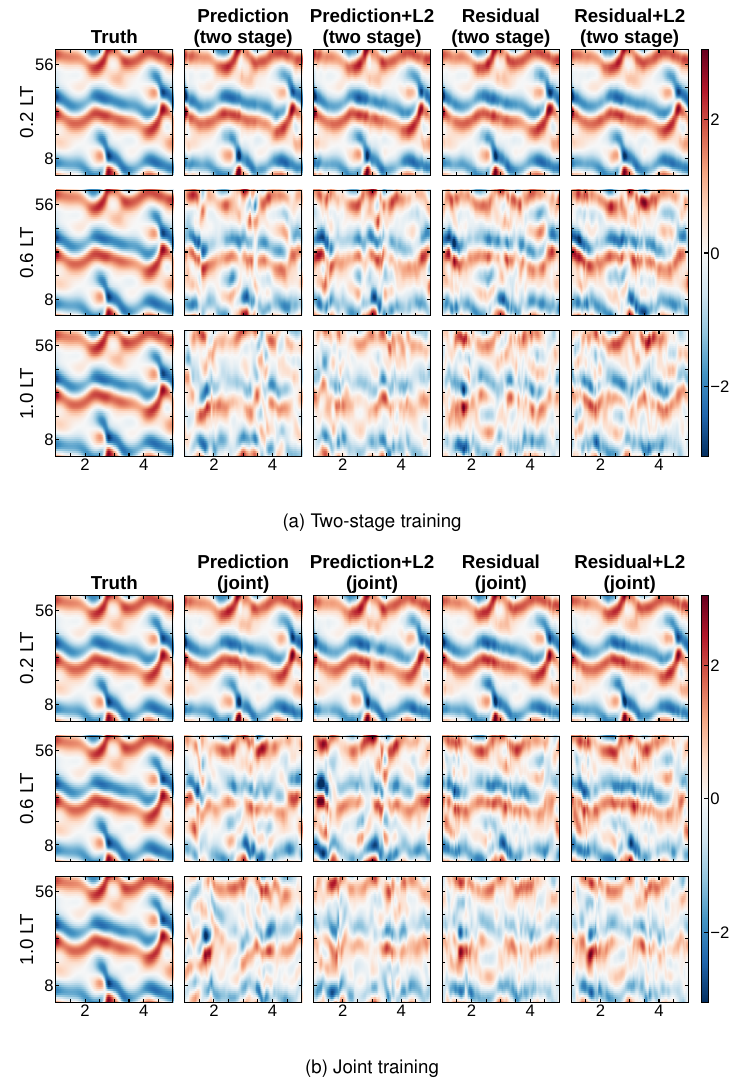}
\caption{Predictions for Kuramoto--Sivashinsky at latent dimension \(16N_0\) (top: two-stage training; bottom: joint training).
From left to right, the columns show the reference field, Prediction (the latent-prediction model), Prediction+L2, Residual (the spectral-residual model), and Residual+L2; +L2 adds the operator-norm penalty.
Each row corresponds to one prediction horizon (0.2, 0.6, or 1.0 LT), and predictions are made that far ahead from different observation vectors in the reference trajectory.
Horizontal axis: prediction target time in LT; vertical axis: spatial sample index 0--63 on a domain of length 22; color: field value, with the same color range across all panels}
\label{fig:supp-decoded-ks}
\end{figure}